\documentclass{article}

\usepackage[preprint]{neurips_2026}

\usepackage[utf8]{inputenc} 
\usepackage[T1]{fontenc}    
\usepackage{hyperref}       
\usepackage{url}            
\usepackage{booktabs}       
\usepackage{amsfonts}       
\usepackage{nicefrac}       
\usepackage{microtype}      
\usepackage{xcolor}         
\usepackage{amsmath}
\usepackage{amsthm}
\usepackage{bbm}
\usepackage{graphicx}
\usepackage{thmtools}
\usepackage{thm-restate}

\usepackage{algorithm}
\usepackage{algpseudocode}
\usepackage{wrapfig}

\newcommand\independent{\protect\mathpalette{\protect\independenT}{\perp}}
\def\independenT#1#2{\mathrel{\rlap{$#1#2$}\mkern2mu{#1#2}}}
\newcommand{\E}{\mathbb{E}}
\newcommand{\bP}{\mathbb{P}}

\newcommand{\R}{\mathbb{R}}

\newcommand{\cN}{\mathcal{N}}

\newcommand{\cov}{\textnormal{Cov}}
\newcommand{\var}{\textnormal{Var}}

\newcommand{\ATE}{\textnormal{ATE}}
\newcommand{\IV}{\textnormal{IV}}

\newcommand{\DR}{\textnormal{DR}}
\newcommand{\Pn}{\mathbb{P}_n}
\newcommand{\one}{\mathbbm{1}}

\newtheorem{theorem}{Theorem}

\newtheorem{lemma}{Lemma}
\newtheorem{assumption}{Assumption}
\newtheorem{corollary}{Corollary}

\title{Learning Treatment Effects during Resource Allocation via Priority-Queue Randomization}

\author{%
  JungHo Lee$^1$ 
  \qquad 
  Johnna Sundberg$^1$ \qquad 
  Pim Welle$^2$ \qquad
  Bryan Wilder$^1$ \\
  $^1$Carnegie Mellon University \\ $^2$Allegheny County Department of Human Services \\
  $^1$\texttt{\{junghol, jsundber, bwilder\}@andrew.cmu.edu} \\
  $^2$\texttt{paul.welle@alleghenycounty.us}
}

\begin{document}
\maketitle

\begin{abstract}
     Public service programs often allocate limited resources under uncertainty about their benefits, creating a need for randomization to support credible evaluation. In practice, however, applicants commonly enter waitlists where resources are prioritized toward individuals judged to have higher need through tiered priority queues, making direct randomization difficult. Motivated by this, we develop an experimental design framework for learning treatment effects while treating those most in need where incoming applicants are randomized into priority queues based on their assessed risk scores. Treatments are then provided across queues in priority order and first-in-first-out within queue as budget becomes available. Our contributions are two-fold. First, we characterize what causal effects are identified under this priority-queue allocation. When arrivals are exogenous, treatments are conditionally randomized, and hence standard estimands are identified; when arrivals are endogenous, queue randomization instead provides an instrument for treatment, identifying local treatment effects induced by the queuing process. Second, we develop optimized queue-assignment designs that trade off statistical efficiency against prioritizing higher-need applicants. We show in the process that, despite dependence in treatment assignments induced by the design, usual iid efficiency bounds remain well-justified design objectives. We illustrate the proposed designs using data from a housing allocation program in a large U.S. county.
\end{abstract}

\section{Introduction}\label{sec:intro}
In many public service programs, scarcity is managed not through a one-time allocation decision, but through tiered priority queues. Applicants arrive over time, resources become available over time, and each newly available slot is offered according to a service rule: applicants in higher-priority tiers are served before those in lower-priority tiers, while applicants within the same tier are typically ordered by time spent waiting. Such priority-based waiting systems are widely used, including for child care subsidies \citep{vdoe2024paying, vdoe2025waitlist}, Medicaid home- and community-based services \citep{macpac2020hcbs}, public housing and homelessness assistance \citep{acdhs2020aha}, and related settings such as organ transplantation \citep{colvin2012lung}.

While these programs provide essential services, they often operate under significant uncertainty regarding their causal impact. Rigorous randomized evaluation is the gold standard for resolving this uncertainty, yet direct randomization, such as a lottery for treatment slots, is often administratively or ethically unfeasible in these contexts. Bypassing high-priority applicants to serve those with lower assessed need would violate the program's core mandate. 

We therefore propose an experimental design framework that resolves this tension by randomizing queue assignment rather than treatment allocation. The resulting design problem is substantially more complex because treatment is not assigned independently at a fixed point in time; applicants arrive sequentially, resources become available over time, and each treatment decision depends on the evolving composition of all priority queues. Moreover, the design must preserve the program's priority-based logic. Under our design, applicants are randomized into priority queues with probabilities depending on their risk scores. Once assigned, treatment assignment is governed by existing rules: resources are allocated across queues in priority order and first-in-first-out within each queue. 

Our contributions are two-fold. First, we characterize what causal effects are identified under this dynamic allocation process. We show that when arrivals are exogenous, treatment is conditionally randomized and standard treatment effects can be recovered; when arrivals are endogenous, queue randomization instead provides an instrument for treatment, identifying local effects induced by the queuing process. Second, we develop optimized queue-assignment designs that balance statistical efficiency with the prioritization of high-need applicants. This design problem is difficult because the policymaker controls queue-assignment probabilities rather than treatment propensities directly, and because treatment assignments are dependent across applicants through the queuing process. We formulate policies that maximize statistical efficiency subject to a utility constraint ensuring that high-need individuals remain likely to be prioritized.


\paragraph{Relevant literature: } 
Our paper builds on \citep{wilder2025learning} who study how to design randomized allocation rules that balance two competing goals: prioritizing individuals with greater need and learning treatment effects. We extend this perspective to settings where treatment is instead allocated through priority-based rules. Our work is also related to the market-design literature on dynamic waitlists and priority-based allocation. This literature studies settings in which agents and/or objects arrive over time, and evaluates mechanisms according to objectives such as welfare, fairness, and incentives to accept or defer offers \citep[e.g.,][]{thakral2019matching, arnosti2020design, schummer2021influencing, leshno2022dynamic, murra2024public}. Our focus is not on characterizing the welfare or incentive properties of a waitlist mechanism per se, but on designing randomization within such mechanisms to enable causal inference while preserving priority-based allocation goals. 


A separate literature studies experimental design and policy learning under capacity constraints \citep{yamin2025dependent, zhou2023mind, tang2023learning, sun2021treatment}. These papers share our interest in scarce resource allocation, but typically choose allocation policies directly, whereas in our setting treatment is induced by a subsequent queuing process. Closest to our inferential setting is perhaps \citet{chaisemartin2020estimating}, who study causal estimation under randomized waitlists induced by limited resources. There, applicants' ranks are randomized and offers proceed down the list until seats are filled, which conflicts with the priority structure we seek to preserve.

\section{Basic setup and queue-randomizing design}\label{sec:setup_and_queue}

We consider a setting in which $n$ units arrive into our system over a fixed enrollment horizon. Each unit has an arrival time $A \in [0,\tau]$ and observed covariates $X \in \mathbb R^p$. Treatments are distributed at discrete allocation periods $t \in \{1,\ldots,\tau\}$, where $\tau \in \mathbb{N}$ is fixed. We use $Z \in \{0,1\}$ to denote terminal treatment status, where $Z=1$ means that the unit is treated at some point before the end of the allocation horizon and $Z=0$ otherwise. 
We adopt the potential outcomes framework for causal inference \citep{neyman1923, rubin1974estimating} and write $Y^0,Y^1 \in \mathbb R$ as the (stationary) potential outcomes under no treatment and treatment, respectively. 
We therefore posit that we have $n$ iid draws $O = (A,X,Y^0,Y^1)$ from some joint distribution $\mathcal P$. We make the standard \textit{consistency} assumption throughout that the observed $Y$ is $Y = Y^z$ if $Z=z$. 

Unlike idealized settings where treatment is determined independently for each individual, many public service settings allocate scarce resources through tiered priority queues. There are $K$ queues, indexed by $\{1,\ldots,K\}$, where lower-indexed queues have higher priority. At each individual's arrival time $A$, they are assigned a queue $Q \in \{1,\ldots,K\}$ based on their observable covariates $X$. At the end of each allocation period, a budget $b$ becomes available, representing $b$ available treatment slots, such as housing vacancies. In the formal notation below, we write this fixed period-specific budget as $b_{t,n}$ to allow treatment capacity to vary across periods and with the system size $n$, and $b_n \equiv (b_1,\ldots,b_{\tau,n})$ for the entire sequence of budgets. At the end of period $t$, the policymaker considers all units who have arrived by that period and have not yet been treated. Treatment is offered first to individuals in the top-priority queue $1$, then to individuals in queue $2$, and so on, until the period-$t$ budget $b_{t,n}$ is exhausted. Within each queue, treatment is first-in, first-out (i.e., treatment is offered in order of arrival time). Units who are not treated in period $t$ remain eligible for treatment in later periods, which therefore yields a waitlist structure.

We study a class allocation policies where the policymaker chooses which queue an individual with covariates $X$ is assigned to via a (potentially randomized) policy $\theta(X)$, where $\theta(X) = (\theta_1(X),\ldots,\theta_K(X))^\top \in \Delta_K$ denotes the vector of queue assignment probabilities from which the queue $Q \in \{1,\ldots,K\}$ is drawn (throughout, $\Delta_K \equiv \{w \in \R_+^K : \sum_{k=1}^K w_k = 1\}$ denotes the simplex). 
Since $\theta$ depends in $X$, 
the policymaker can offer different assignment probabilities depending on individual characteristics. At one extreme, higher-risk individuals might be deterministically assigned to higher-priority queues, as is typically the case in practice. At the other end, $\theta$ could be uniform for all individuals (effectively implementing an RCT). In between, higher risk individuals can be offered a greater probability of placement in top-priority queues. The choice of queue for an individual induces the probability that they will actually receive treatment. Let $\mathbb P_{\theta,b_n}$ be the distribution of the treatment induced by our design. We use $\tilde\pi_n(x,q;\theta,b_n) \equiv \bP_{\theta,b_n}(Z=1\mid X=x,Q=q)$ to denote the \textit{queue-conditional propensity score} and  $\pi_n(x;\theta,b_n) \equiv \bP_{\theta,b_n}(Z=1 \mid X=x)$ to denote the \textit{marginal propensity score} conditional just on covariates $X$.

Queue-based structure is needed to manage the operational complexity of real-world services, where applicants and resources both arrive online. However, it creates substantial complexity for inference of treatment effects because whether an individual is offered treatment depends on how the arrival of other individuals and resources flow through the queuing system. This stands in stark contrast to many other allocative settings, where a simple cutoff for eligibility enables designs like regression discontinuities. Our goal is two-fold. First is to determine \textit{which causal effects are identified under this allocation mechanism} given the complexity of the process. Second is develop experimental designs -- \textit{how can a policymaker select a randomized policy} $\theta$ enable accurate estimation of treatment effects while still allocating treatment to those most in need?



\section{What estimands can be identified?}\label{sec:identification}
The most basic question is what causal effects can be identified under what conditions for queue-based allocation systems. So far, we have not imposed any assumptions on how the arrival times $A$ relate to the other variables. This turns out to be a critical distinction because the arrival time distribution determines how resources are prioritized between otherwise-similar individuals.   

The simplest setting would be if the arrival time were \textit{independent} of the potential outcomes, at least conditional on observables:  $A \independent (Y^0, Y^1) \mid X$. In this setting, arrival times create a ``natural experiment" of their own, apart from any explicit randomization performed by the policymaker. That is, individuals who arrive when there happens to be more budget available may be more likely to receive treatment than otherwise-similar individuals who arrive when demand is higher and the system more congested. Formalizing this intuition, we show:

\begin{theorem}\label{thm:identification_exo}
    Suppose that the arrivals satisfy $A \independent (Y^0, Y^1) \mid X$. Then, treatment assignments are conditionally randomized: $ Z \independent (Y^0, Y^1) \mid X$. Moreover, suppose that every level of the covariates $X$ satisfies positivity: there exist $\gamma > 0$ and $n_0 \in \mathbb N$ such that for every $n \geq n_0$, $\gamma \leq \pi_n(X;\theta,b_n) \leq 1-\gamma$. Then, any functional of the potential outcome distribution $\bP(Y^z \mid X=x)$ is identified. For example, the average treatment effect (ATE) $\psi_{\ATE} \equiv \E(Y^1-Y^0)$ is identified.  \label{theorem:exog-arrivals}
\end{theorem}

This demonstrates that assuming exogeneity of arrival times suffices to identify many causal effects of interest simply by comparing individuals who are vs are not treated (conditional on covariates). Such comparisons have been used in a number of applied program evaluation settings, and Theorem  \ref{theorem:exog-arrivals} helps clarify the assumptions on the arrival process necessary for this strategy to be valid (which have not been discussed in previous resources for practitioners, e.g.\ \citet{frechtling2010userfriendly}). 

Making these assumptions explicit reveals that they are potentially stringent when arrivals may be correlated with underlying need. For example, consider a public housing program. It could be the case that more people apply for public housing assistance in time periods where outcomes are worse (e.g., because of worse overall economic circumstances), in ways that may not be fully captured by the covariates. In such a case, the naive analysis suggested by Theorem \ref{theorem:exog-arrivals} would suffer from confounding. Hence, we now turn to the development of identification strategies which do not rely on the assumption $A \independent (Y^0, Y^1) \mid X$. 

Discarding the assumptions of exogenous arrivals substantially complicates the identification of treatment effects. Even if the policymaker explicitly randomizes which queues individuals are placed into via the choice of $\theta$, this does not suffice for treatment assignments to be randomized: endogeneity in the arrival process can still confound treatment assignment by making individuals who arrive at particular times (and have particular potential outcomes) be more likely to receive treatment. 

To resolve these difficulties, we propose an identification strategy which treats the random assignment of the queue as an instrumental variable (IV). Since the queue assignment can be explicitly randomized by the policymaker, standard IV assumptions will be satisfied by design.\footnote{Exclusion is justified here because queue assignment is an administrative priority label whose only operational role is to determine the order in which applicants are offered treatment, and hence whether they receive treatment before the end of the allocation horizon. Although queue assignment may affect waiting time, our stationary potential outcome formulation assumes that outcomes depend only on terminal treatment receipt, not on the exact timing. This restriction would fail if queue assignment changed other aspects of service delivery or if outcomes depended directly on treatment timing.} Following the literature on formula instruments \citep{borusyak2025design}, we propose to instrument $Z$ using the difference between a given individual's \textit{expected} treatment probability pre-randomization and their \textit{realized} treatment probability post-randomization:
\begin{align*}
    r_n(X, Q) \equiv \tilde\pi_n(X,Q;\theta,b_n)-\pi_n(X;\theta,b_n).
\end{align*}
$r_n$ accounts for the fact that individuals with particular covariates $X$ might be more likely ex-ante to be placed into queues with high treatment probabilities by recentering relative to $\pi_n(X;\theta,b_n)$. 



\begin{lemma}\label{lem:valid_iv}
   Under our design, $r_n(X,Q)$ satisfies conditional IV-unconfoundedness, exclusion, and the induced treatment statuses satisfy monotonicity. 
\end{lemma}

A natural causal estimand is therefore the IV ratio:
 \[\beta_{\IV} \equiv
    \frac{\E\{r_n(X,Q)Y\}}{\E\{r_n(X,Q)Z\}}
    =
    \frac{\E\big[\{\tilde\pi_n(X,Q;\theta,b_n)-\pi_n(X;\theta,b_n)\}Y\big]}
    {\E\big[\{\tilde\pi_n(X,Q;\theta,b_n)-\pi_n(X;\theta,b_n)\}Z\big]},
\]
given the relevance assumption $\E\{r_n(X,Q)Z\} \neq 0$. However, it is not apriori clear exactly what treatment effect $\beta_{\IV}$ targets. Generally, IVs are interpretable as identifying a treatment effect on a latent complier population who are treated only ``because" of the instrument. Here, this complier population is induced by the queuing process. We next derive an explicit interpretation of $\beta_{\IV}$ as a particular weighted combination of local average treatment effects (LATEs).

Formally, for any two queues $k < \ell$, let $\mathcal{C}_{k,\ell}$ denote the set of individuals who (because of their arrival times) would be treated if randomized to queue $k$ but \textit{not} if randomized to queue $\ell$. It turns out $\beta_{\IV}$ corresponds to a weighted average of treatment effects for such latent complier populations:

\begin{theorem}\label{thm:arrival_endo_late}
Let $\psi_{k,\ell}(X) \equiv \E(Y^1 - Y^0 | X, \mathcal{C}_{k,\ell})$ and $w_{k,\ell}(X) \equiv \theta_k(X)\theta_\ell(X)(\tilde\pi_n(X,k;\theta,b_n) - \tilde\pi_n(X,\ell;\theta,b_n))^2$. Then,
\[ 
    \beta_{\IV} = \frac{\E\{\sum_{k < \ell}w_{k,\ell} (X)\psi_{k,\ell}(X)\}}{\E\{\sum_{k < \ell}w_{k,\ell}(X)\}}.
\]
\end{theorem}
The weights $w_{k, \ell}(X)$ reflect that particular LATEs $\psi_{k,\ell}(X)$ are weighted more heavily if there is a larger probability of individuals with covariates $X$ being randomized to queues $k$ and $\ell$ and if queues $k$ and $\ell$ induce substantially different treatment probabilities. That is, such designs target individuals whose treatment status is more likely to change because of randomization.  

In practice, $\beta_{\IV}$ might be hard to interpret if there is little basis for identifying which individuals are treated only at specific levels of priority. To provide a simplified estimand for interpretation, and later for experimental design, we also study the partially linear IV model 
\begin{equation}\label{eq:plm_itt}
    Y = \psi_Z Z + g(X) + U, \qquad \E(U \mid X,Q)=0,
\end{equation}
studied by, e.g., \cite{robinson1988root} and \cite{chernozhukov2018double}. This model estimates a single homogeneous treatment effect $\psi_Z$, although the covariates are still allowed to have a separate (nonparametric) effect on the outcome via $g(X)$. For practitioners interested in estimating such an ``average" effect, we show that a scaled version of $r_n$ remains the uniquely best choice of instrument for attaining the usual iid efficiency bound \citep{chamberlain1992efficiency}:
\begin{theorem}\label{thm:pliv_instrument}
    Let $\sigma(X) \equiv \E(U^2|X)$. Then, $r_n(X,Q)/\sigma(X)$, up to multiplicative scaling, uniquely attains the iid efficiency bound for the partially linear IV model \eqref{eq:plm_itt}.
\end{theorem}
Intuitively, $r_n(X,Q)$ captures the randomized first-stage variation generated by queue assignment after partialing out $X$. Scaling by $1/\sigma(X)$ then prioritizes covariate regions with less residual outcome noise, yielding the usual inverse-variance weighting logic.

\section{How should experiments be designed?}\label{sec:design}
So far, we have treated the queue assignment policy $\theta(X)$ as fixed. Next, we turn to the \textit{design} of experiments that use randomization to queues. In particular, we consider policymakers who wish to balance learning treatment effects against the loss that randomization creates with respect to the originally desired prioritization. Following \cite{wilder2025learning}, we formalize this by modeling the policymaker as having a utility $u(X)$ for offering treatment to a participant with covariates $X$. For example, $u$ could be a risk score, representing a preference for allocating treatment to higher-risk participants. If the policymaker has no desire for experimentation, they would simply assign queues deterministically based on $u$, e.g., assigning the highest-risk applicants to the highest priority queue, and so on. However, without randomization, it becomes difficult to credibly infer treatment effects (particularly without assuming exogeneity of arrivals). On the other hand, uniformly random assignments would maximize power to estimate treatment effects but entirely ignore $u$. 

Our goal is to design policies $\theta$ which lie on the Pareto curve between these extremes, allowing policymakers to optimally trade off between the two goals. \cite{wilder2025learning} showed that in the case of single-shot treatment decisions, this can be reduced to a convex optimization problem that trades off expected utility with the optimal variance (efficiency bound) for estimating the treatment effect. However, the dynamic setting with queue-based allocation creates substantial new challenges.  First, we optimize over the queue allocation probabilities $\theta$ instead of directly over the treatment propensities $\pi$; formulating an optimization problem requires a tractable structure for the map $\theta \mapsto \pi$.  Second, in our setting, treatment assignments are not independent across individuals: the queue process creates correlations since individuals ``compete" for fixed resources. Without iid data, it is unclear whether standard efficiency bounds apply. Third, particularly under endogeneous arrivals, the treatment effects that we can estimate are substantially more complicated than a simple average treatment effect and it is not clear how to derive a tractable variance proxy for optimization. 

To address the first and second challenge, we describe a set of assumptions on the queuing process under which we can formulate tractable proxy objectives for optimization. We emphasize that these assumptions are \textit{not} required  for valid identification of treatment effects: our results in the previous section hold regardless. Instead, we view them as a useful approximation that can be imposed at the design stage in order to guide the planning of an experiment. All asymptotic statements are taken with respect to a growing experiment size $n \to \infty$, holding fixed the number of queues $K$, the number of allocation periods $\tau$, the queue-assignment policy $\theta$, and the arrival-time distribution.

\begin{assumption}[Fixed queue proportions]\label{assm:fixed_queue}
    $\E\{\theta(X)\} = p \equiv (p_1,\ldots,p_K)^\top \in \Delta_K$.
\end{assumption}

\begin{assumption}[Stable and non-negligible budget]\label{assm:stable_budget}
There exists $\beta \in (0,1)$ such that, as $n \to \infty$,
\[
    \max_{1 \leq t \leq \tau} \bigg|\frac{1}{n}\sum_{s \leq t} b_{s,n} - \beta \mathbb P(A \leq t)\bigg|\to 0.
\]
\end{assumption}


\begin{assumption}[Independent arrivals]\label{assm:arrivals_indep_covariates}
The arrivals are independent of covariates, i.e., $A \independent X$.
\end{assumption}

We view Assumption \ref{assm:fixed_queue} as a mild condition that is enforceable by-design: it says that each queue should capture some fixed proportion of applicants in-expectation. This is true of many empirical examples, where, e.g., a queueing system with 3 levels is designed so that the top queue corresponds to the 1/3 highest-risk individuals. 
Assumption \ref{assm:stable_budget} says that the budget is released over time in proportion to enrollment and that, asymptotically, treatment capacity is a constant fraction of the experiment size. If the fraction treated vanishes, then there is not an appropriate asymptotic sense in which the amount of the data observed increases with $n$.

Assumption \ref{assm:arrivals_indep_covariates} is perhaps the most substantive, and says that arrival times do not sort systematically according to covariates. This can be viewed as optimizing under a stationarity condition, where the arrival distribution does not change significantly over time. We view this as a reasonable approximation for the design stage since the policymaker likely does not have the ability to predict with any precision how the future arrival distribution will change. Additionally, without such an assumption, the relevant asymptotics are hard to define since the queuing system and corresponding treatment propensities need not converge as the number of individuals grows. 

\paragraph{Correspondence between queues and treatment propensities: }Under these conditions, we address the first challenge by showing that the policy $\theta$ induces a tractable mapping to the resulting treatment propensities $\pi$. The relevant object for later analysis is the \emph{asymptotic propensity score}
\begin{equation}\label{eq:asymp_propensity}
    \pi(x;\theta)
    \equiv
    \lim_{n\to\infty}\pi_n(x;\theta,b_n)
    =
    \lim_{n\to\infty}\sum_{k=1}^K \tilde\pi_n(x;\theta,b_n)\theta_k(x).
\end{equation}
We show $\pi(x; \theta)$ takes a specific form:
\begin{theorem}[Design-induced asymptotic propensity score]\label{thm:propensity}
    Suppose Assumptions \ref{assm:fixed_queue}--\ref{assm:arrivals_indep_covariates} hold. Let 
    $c_k \equiv \sum_{j=1}^k p_j$ with $c_0 \equiv 0$. For each $k$, define
    \begin{equation*}
         \alpha_k(\beta,p) \equiv
         \begin{cases}
             \{(\beta-c_{k-1})_{+}-(\beta-c_k)_{+}\}/p_k, & p_k > 0, \\
             0, & p_k = 0,
         \end{cases}
     \end{equation*}
    where $u_+ \equiv \max\{u,0\}$. Then, for a.e. $x$, $\pi(x;\theta) = \sum_{k=1}^K \alpha_k(\beta,p)\theta_k(x)$.
\end{theorem}

Intuitively, this corresponds to modeling individuals as having a fixed probability of treatment conditional on their queue assignment, with the final treatment propensity being a mixture of these queue-specific probabilities. Thus the asymptotic propensity score is affine in $\theta$, which is the key structural property that makes the design problem tractable.


\paragraph{Estimation with non-iid assignments:} We address the second challenge by showing that, even with dependent treatment
allocations, estimators based on the usual iid efficient influence function (EIF) or orthogonal estimating equation representations \citep{kennedy2024semiparametric} have the same leading asymptotic variance as in the iid case. We start with the case of exogenous arrivals, focusing on the standard ATE estimand. Let $\mu_z(x) \equiv \E(Y\mid Z=z,X=x)$ and $\Pn f \equiv n^{-1}\sum_{i=1}^n f(O_i)$ denote the empirical average. A canonical EIF-based estimator is the doubly-robust estimator 
\begin{equation*}
 \widehat\psi^{\mathrm{DR}}_{\ATE} 
 \equiv \Pn \bigg[\widehat\mu_1(X) - \widehat\mu_0(X) + \frac{Z}{\pi_n(X;\theta,b_n)}\{Y - \widehat\mu_1(X)\} - \frac{1-Z}{1-\pi_n(X;\theta,b_n)}\{Y - \widehat\mu_0(X)\}\bigg].
\end{equation*}
Its asymptotic variance under our treatment allocation scheme can be characterized as follows:
\begin{theorem}[Asymptotic variance of $\widehat\psi^{\mathrm{DR}}_{\ATE}$]\label{thm:asymp_var_DR}
Suppose the conditions in Theorem \ref{thm:identification_exo} and Assumptions \ref{assm:fixed_queue}--\ref{assm:arrivals_indep_covariates} hold. Further, assume $\widehat\mu_z$ is constructed from an independent sample $\mathcal D$ such that $\|\widehat\mu_z-\mu_z\|_{L_2(\bP)}=o_{\bP}(1)$ for $z\in\{0,1\}$. If there exists $m < \infty$ such that  $\E\{(Y^1)^2 +(Y^1)^0\} \leq m$, 
then conditional on $\mathcal D$,
\[
\lim_{n\to\infty} n\var(\widehat\psi^{\mathrm{DR}}_{\ATE})
= 
\E\bigg\{\frac{\var(Y \mid Z=1, X)}{\pi(X;\theta)}+
\frac{\var(Y \mid Z=0, X)}{1-\pi(X;\theta)}\bigg\}
+\var\big(\mu_1(X)-\mu_0(X)\big).
\]
\end{theorem}

This result has two implications. First, it characterizes when we can analyze the data from the queue-based experiment \textit{as if} the treatment assignments were iid and still obtain valid confidence intervals based on asymptotic normality (Appendix \ref{sec:analysis}). Second, the variance above is exactly the semiparametric efficiency bound for estimating the ATE from iid data \citep{hahn1998role, imbens2009recent}. Thus, the same quantity gives the accuracy with which treatment effects can be estimated even when treatment assignments are correlated via the queuing process. This will allow us to use the efficiency bound to instantiate an objective for optimizing over experimental designs.

The key intuition behind Theorem \ref{thm:asymp_var_DR} is that the dependence induced queuing systems, though complex, is structured. There are two sources of randomness: the queue composition and the arrival order within partially served queue-period cells. Under our assumptions, queue composition is stable in the sense that changing one other unit's queue changes each queue count by at most one, and treated queue counts by at most a constant. Conditional on the realized design variables, treatment is deterministic outside partially served cells, while treatment within such cells behaves like sampling without replacement from a cell of size order $n$. Thus pairwise treatment dependence is at most order $n^{-1}$. Importantly, this dependence does not affect the leading oracle term: after decomposing the DR estimator into an average of oracle influence functions plus a nuisance-estimation remainder, the dependence enters only through the nuisance remainder, where it is multiplied by nuisance estimation errors. Because $\widehat\mu_z$ are consistent, the remainder contributes $o_\bP(1)$ to the scaled variance, leaving the usual influence-function variance.

\paragraph{Estimation with endogenous arrivals:} we now show an analogous result holds for the case of endogenous arrivals, when randomization at the queue level provides an instrumental variable for the treatment. As discussed above, the definition and interpretation of treatment effects is more complex in this setting. The fully nonparametric characterization of the IV estimand show that it is a weighted combination of LATEs where the weights depend themselves on the choice of instrument (i.e., the policy $\theta$ we are now optimizing over). To provide a tractable objective for the design stage, we suggest optimizing with respect to estimation efficiency for the treatment effect in the partially linear IV (PLIV) model, which provides a natural counterpart to estimating a homogeneous average treatment effect. 
We show an analogous result to Theorem \ref{thm:asymp_var_DR} holds for an estimator  obtained by solving an empirical moment equation induced by the instrument in Theorem \ref{thm:pliv_instrument}:
\[
    \widehat\psi_Z^\star
    \equiv
    \frac{\Pn\left[\frac{r_n(X,Q)}{\widehat\sigma(X)}\{Y-\widehat m(X)\}\right]}
         {\Pn\left[\frac{r_n(X,Q)}{\widehat\sigma(X)}\{Z-\pi_n(X;\theta,b_n)\}\right]}, \qquad m(x) \equiv \E(Y \mid X=x).
\]
Theorem~\ref{thm:asymp_var_plm} shows that, under our design, the asymptotic variance of $\widehat\psi_Z$ matches the usual efficiency bound expression evaluated at the limiting instrument $\alpha_Q(\beta,p)-\pi(X;\theta)$ \citep{chamberlain1992efficiency}.

\begin{theorem}[Asymptotic variance of $\widehat\psi_Z^\star$]\label{thm:asymp_var_plm}
Suppose Assumptions \ref{assm:fixed_queue}--\ref{assm:arrivals_indep_covariates} and the model \eqref{eq:plm_itt} hold. Further, assume $\widehat \sigma$, $\widehat m$ are constructed from an independent sample $\mathcal D$ such that $\|\widehat\sigma - \sigma\|_{L_2(\bP)} = o_\bP(1)$ and $\|\widehat m-m\|_{L_2(\bP)} = o_\bP(1)$. If $\E\left[\{\alpha_Q(\beta,p)-\pi(X;\theta)\}^2\right] \geq \gamma >0$, then conditional on $\mathcal D$,
\[
    \lim_{n\to\infty} n\var(\widehat\psi_Z^\star)
    = \left(\E\left[\frac{\{\alpha_Q(\beta,p)-\pi(X;\theta)\}^2}{\sigma(X)}\right]\right)^{-1}.
\]
\end{theorem}

\paragraph{Multiobjective optimization:} In either setting (exogenous or endogenous arrivals), our goal now is to select the queue assignment probabilities $\theta$ to optimally trade off between the participant's utility $u(X)$ and statistical power to estimate treatment effects.
We use $\mathbb{V}_\psi(\theta, \eta)$ to denote the asymptotic variance of an estimate of treatment effect $\psi$ given queue assignment probabilities $\theta$ and nuisance functions $\eta$. $\eta$ comprises terms in the objective that are not knowable ahead of running the experiment. For example, for the ATE estimand above, $\eta = (\var(Y|Z = 1, \cdot), \var(Y|Z = 0, \cdot))$ since the conditional variances of the outcomes typically cannot be estimated precisely before running an experiment. Similarly, for the PLIV case, $\eta = \sigma(\cdot)$. We stipulate that such nuisances lie within an uncertainty set $\mathcal{U}$ and solve a robust optimization problem of the form 
\begin{align*}
    \min_{\theta: \mathcal{X} \to \Delta_k} \max_{\eta \in \mathcal{U}}  \mathbb{V}_\psi(\theta, \eta)  \text{ s.t. } 
    \,\, \E[\theta(X)] = p; \,\,\E[\pi(X; \theta)u(X)] \geq c,
\end{align*}
where the 
first requires that the queue assignment probabilities $\theta$ satisfy the budget constraint, and the second imposes the expected utility of the resulting treatment assignments be at least $c$. By adjusting the value of $c$, we can trace out the entire Pareto curve between treatment effect variance and expected utilities, allowing the policymaker to select a point that represents their desired balance. Note that since $\pi$ is affine in $\theta$, all of the constraints are linear with respect to $\theta$. 

As a general approach to constructing the uncertainty set $\mathcal{U}$, we assume that the nuisances are bounded, i.e., $\var(Y|Z,X) \leq \delta$ for some constant $\delta$ for the ATE or $\sigma(X) \leq \delta$ for the PLIV model. Generically, we represent this as an uncertainty set of the form $\mathcal{U} = \{\eta: ||\eta||_\infty \leq \delta\}$. If the policymaker happens to have additional information (e.g., a tighter bound on one of the conditional variance functions), our techniques can be naturally extended. For uncertainty sets that have such box-constraint style structure, we obtain tractable simplifications of the inner max:
\begin{theorem}
    For the ATE case, minimizing $\max_{\eta \in \mathcal{U}} \mathbb{V}_{\ATE}(\theta, \eta)$ is equivalent to minimizing $\E\left\{\frac{1}{\pi(X;\theta)} + \frac{1}{1 - \pi(X;\theta)}\right\}$, which is convex in $\theta$. For the PLIV case, minimizing $\max_{\eta \in \mathcal{U}} \mathbb{V}_{\mathrm{PLIV}}(\theta, \eta)$ is equivalent to maximizing $\E\left[\{\alpha_Q(\beta,p)-\pi(X;\theta)\}^2\right]$, which is concave in $\theta$. 
\end{theorem}

One final complication is that since $\pi$ is affine in $\theta$, there may not be a unique solution to the above optimization problems. We resolve this by augmenting the objective with a regularizer $\kappa R(\theta)$ where $R$ is strictly convex and $\kappa > 0$. Experimentally, we take $R$ to be the (negative) entropy to tie-break in favor of designs with greater randomization at the queue level. The final result is a strictly convex minimization (or strictly concave maximization) with respect to $\theta$.

To compute the resulting policies in practice, we work with the dual formulation of the constrained population problem that yields a tractable finite-sample analogue. We defer the explicit sample optimization problem and additional details to Appendix~\ref{sec:opt_detail}.

\section{Numerical experiments}\label{sec:experiments}

\begin{figure}
    \centering
    \includegraphics[width=\textwidth]{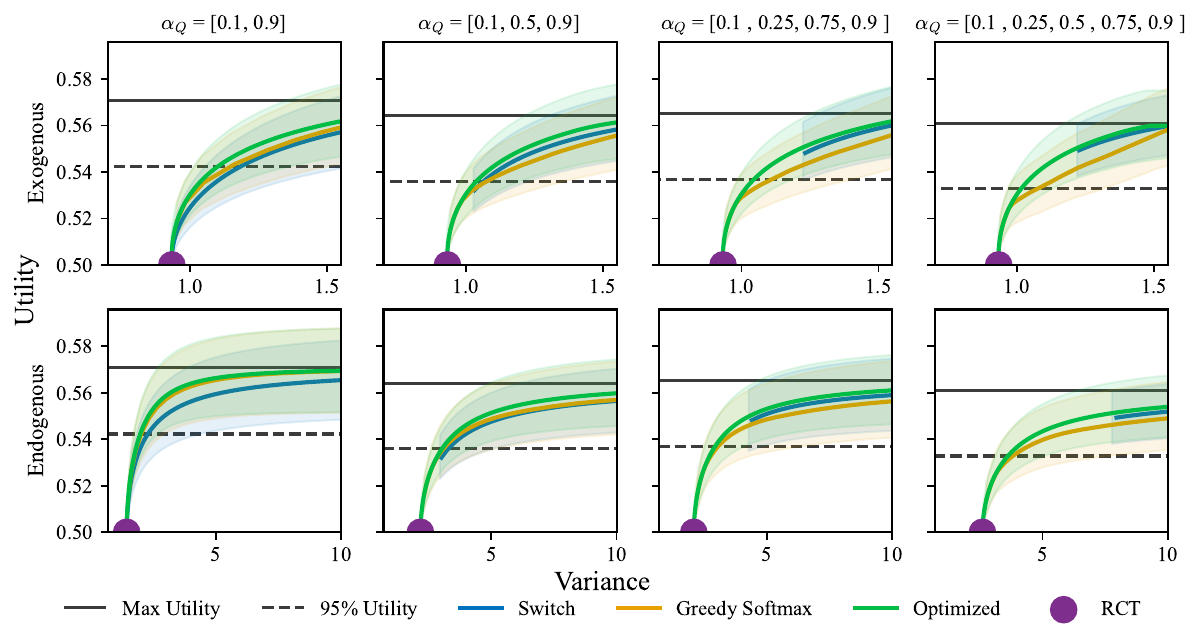}
    \caption{Pareto frontier of variance and utility under our optimized designs (green) versus two heuristic designs (orange, blue) and an RCT baseline (purple) under exogenous arrivals (top) and endogenous arrivals (bottom) across different queue configurations with a 50\% treatment budget. Queues are equally sized but vary in number, with treatment probabilities at the top of each column. The optimized designs achieve higher efficiency and utility across all configurations compared to the heuristic methods. Generally, exogenous designs are more efficient than endogenous designs.}
    \label{fig::sample_efficiency}
\end{figure}

We evaluate the sample efficiency and robustness to arrival endogeneity of our proposed designs using data from a public housing allocation program in a large U.S. county. 
The program assigns individuals to priority queues using a predictive risk score of homelessness, $h(X)$, built from administrative records, and allocates housing units using exactly the kind of first-in, first-out assignment mechanism studied here. The dataset contains 2,317 service-eligible individuals and spans the last 10 years.
Our outcome of interest is an indicator of whether an individual experiences homelessness in the six months following unit assignment. 
Throughout, we set utility $u(X) = h(X)$ to reflect a preference for assigning units to those with the highest predicted risk of homelessness.
As only outcomes under the assigned treatment are observed, we use the predictive risk scores to simulate potential outcomes with a homogeneous treatment effect.
We provide further details on the data and synthetic data-generating process in Appendix~\ref{app:data}.

\subsection{Sample Efficiency}

To evaluate the efficiency of our proposed estimators, we consider equally-sized queues under a 50\% treatment budget across various queue treatment probabilities. We include results under alternative queue configurations and budgets in Appendix~\ref{app:sample_efficiency}.
We then solve the optimization problem across a range of utility constraints for both the exogenous and endogenous arrivals designs.
Figure~\ref{fig::sample_efficiency} shows the Pareto frontier between utility and asymptotic variances in Theorems~\ref{thm:asymp_var_DR} and \ref{thm:asymp_var_plm}.
We compute 95\% confidence bands using the multiplier bootstrap with 10,000 replicates \citep{kennedyNonparametricCausalEffects2019, vandervaartWeakConvergenceEmpirical1996}. 
We compare our proposed designs to two heuristics for constructing queue assignment probabilities: 
(1) \textit{Switch} assigns individuals to queues deterministically via $u(X)$, then randomizes each assignment via switching probabilities to adjacent queues weighted by relative queue size, (2) \textit{Greedy Softmax} sequentially constructs assignment probabilities from highest- to lowest-priority queues using a softmax-style score over remaining utility.
In order to enforce queue proportions, Greedy Softmax requires access to the full sample.
Both methods include a parameter governing assignment determinism, which we vary to compute the Pareto curve.
Additional implementation details are provided Appendix~\ref{app:sample_efficiency}.

Figure~\ref{fig::sample_efficiency} reveals a favorable concave trade-off between utility and statistical efficiency (variance) under our proposed designs.
Both designs achieve high utility at modest cost to variance relative to an RCT.
The endogenous design is generally higher variance than the exogenous design,  
and becomes increasingly so as the number of queues increases.
Our optimization approach generally outperforms both heuristics, achieving consistently high statistical efficiency and utility across all queue configurations. 
The Switch heuristic's restriction to neighboring queues means that it is only capable of generating a portion of the full tradeoff curve, particularly when the number of queues is large. For a smaller number of queues, its performance is dominated by the optimized design. On the other hand, the performance of Greedy Softmax deteriorates as more queues are added.
As the number of queues increases, more individuals fall near queue boundaries, and the greedy allocation spreads their residual mass across an increasing number of lower queues rather than concentrating mass in the adjacent queue.
The optimized design avoids both failure modes by jointly solving for assignment probabilities across all queues simultaneously.

\subsection{Robustness to Arrival Endogeneity} 
\begin{wrapfigure}{r}{0.45\linewidth}
  \includegraphics[width=\linewidth]{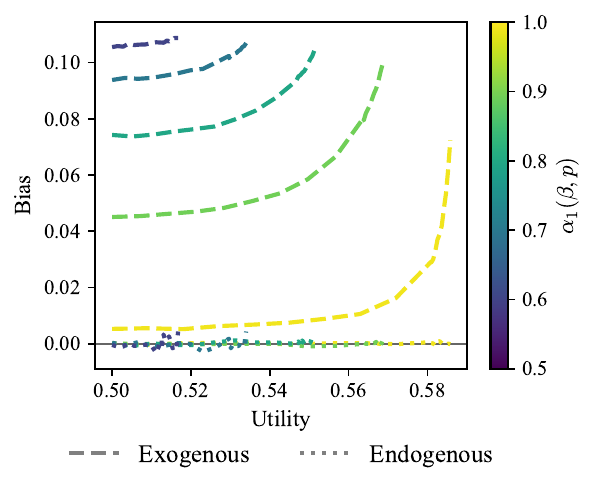}
  \caption{Bias under arrival endogeneity for two queues with varying treatment probabilities. The endogenous design (dotted) remains unbiased; the exogenous design (dashed) shows substantial bias at high utility levels and when treatment probability in the highest-priority queue is low.}
  \label{fig:bias}
\end{wrapfigure}
Although generally more sample efficient, the design under exogenous arrivals may lead to biased estimates if arrival times are not conditionally independent of potential outcomes. 
To demonstrate, we introduce a synthetic confounder $U$ to induce arrival endogeneity. 
We then estimate $\widehat{\psi}_Z$ and $\widehat{\psi}^{\DR}$ using the known synthetic values and calculate the bias of each estimator, see Appendix~\ref{app:bias} for additional implementation details.
Figure~\ref{fig:bias} shows the average bias of both designs across utility levels for two equally-sized queues under a 50\% budget and varying treatment probabilities. 
For each queue configuration, we solve the optimization problem to obtain queue assignment policies under both designs across a range of utility constraints, and compute bias over 10,000 simulations.
As expected, the endogenous estimator (dotted) is unbiased across all levels of utility.
The exogenous estimator (dashed) exhibits substantial bias at high levels of utility and low treatment probabilities in the highest-priority queue.
As the treatment probability of the highest-priority queue $\alpha_1$ increases, treatment becomes mostly determined  by queue assignment rather than arrival order, thus mitigating arrival endogeneity.

\section{Conclusion}\label{sec:conclusion}
In this paper, we introduced an experimental design framework that integrates randomized queue assignment with the priority-based allocation rules that are standard in public service settings. We first characterized causal identification under two distinct arrival regimes and developed optimized queue-assignment policies that effectively balance statistical efficiency against the utility of prioritizing high-need applicants. Our current work focuses on whether treatment is received at all before the allocation horizon ends. While this is natural when the primary scientific question concerns overall access to treatment, the timing of the intervention may itself be of substantive importance. In such cases, it would be valuable to extend our framework to accommodate time-indexed treatment probabilities or the full distribution of treatment timing. 


Overall, our results show that incorporating randomization into priority-queue allocation schemes offers policymakers a means to learn about program effectiveness while limiting sacrifices in the targeting of the program and accommodating the operational constraints of dynamic settings. While the right balance between objectives is setting-specific, our hope is that expanding the Pareto frontier opens up the possibility for more widespread evaluation and improvement of critical services. 

\paragraph{Acknowledgment:} This material is based upon work supported by the AI Research Institutes Program funded by the National Science Foundation under AI Institute for Societal Decision Making (AI-SDM), Award No. 2229881.

\bibliography{references}
\bibliographystyle{apalike}

\appendix

\section{Optimization details}\label{sec:opt_detail}

Recall that we let $\eta \in \mathcal U$ denote the set of unknown nuisance functions at the time of optimization. Then, our optimization problem is 
\begin{align}
\label{app:eq:minmaxefficiencyobjective}
    \min_{\theta: \mathcal{X} \to \Delta_k} \max_{\eta \in \mathcal{U}}  \mathbb{V}_\psi(\theta, \eta)
    \quad \text{ s.t. } 
    \E[\theta(X)] = p; \,\,\E[\pi(X; \theta)u(X)] \geq c.
\end{align}
It is worth noting that the above optimization problem can be easily extended to include other constraints, such as bounding the maximum difference in treatment probabilities between demographic groups to satisfy fairness objectives. 
Formally, following \citet{wilder2025learning}, if each additional constraint $j$ is of the form $\E\{g_j(\pi(X; \theta), X)\} \leq c_j$, as long as the constraint is convex in $\theta$ it may be added to the optimization problem in (\ref{app:eq:minmaxefficiencyobjective}).

In the sections that follow, we provide the complete sample optimization objective under both the exogenous and arrival settings.
This objective is generally not strictly convex in $\theta$, and so we conclude by discussing the choice of strictly convex regularization terms $R(\theta)$ which may be added to the objective to obtain unique solutions.

\subsection{Optimization under exogenous arrivals}\label{sec:opt_detail_exo}

Recall the asymptotic variance from Theorem~\ref{thm:asymp_var_DR}:
\begin{align*}
\lim_{n \to \infty} n \var(\widehat \psi_{\ATE}^{\DR})= \mathbb E \left\{\frac{\var(Y \mid Z = 1, X)}{\pi(X; \theta)} +\frac{\var(Y \mid Z = 0, X)}{1 - \pi(X; \theta)}\right\} + \var(\mu_1(X) - \mu_0(X)). 
\end{align*}

Under exogenous arrivals, we assume that $\eta = (\var(Y \mid Z = 1, X), \var(Y \mid Z =0, X)) \equiv (\eta_1,\eta_0)$.  It may be the case that the policymaker has estimates and/or reasonable assumptions about these nuisances, for example using historical data to estimate $\var(Y \mid Z =0, X)$ and assuming $\var(Y \mid Z = 1, X)$ is bounded above by some function $h(\var(Y \mid Z =0, X), X)$.
In this case, the policymaker can plug-in these values directly into the optimization problem.
If the policymaker does not provide these values, we assume that the uncertainty set is of the form $\mathcal U = \{ \eta : \|\eta\|_{\infty} \leq \delta\}$, i.e. that all nuisances are bounded above.
In this case, the objective in (\ref{app:eq:minmaxefficiencyobjective}) under the limiting asymptotic variance in Theorem~\ref{thm:asymp_var_DR} may be rewritten as:
\begin{align*}
&\min_{\theta: \mathcal X \to \Delta_k} \max_{\eta \in \mathcal U} \mathbb E \left\{\frac{\eta_1}{\pi(X; \theta)} + \frac{\eta_0}{1 -\pi(X; \theta)}\right\} + \var(\mu_1(X) - \mu_0(X)) \\ 
&= \min_{\theta: \mathcal X \to \Delta_k} \delta \mathbb E \left\{\frac{1}{\pi(X; \theta)} + \frac{1}{1 -\pi(X; \theta)}\right\}.
\end{align*}
And hence, the full optimization problem simplifies to
\begin{align*}
&\min_{\theta: \mathcal X \to \Delta_k} \mathbb E \left\{\frac{1}{\pi(X; \theta)} + \frac{1}{1 - \pi(X; \theta)}\right\}  \\
&\pi(X; \theta) = \sum_{k = 1}^{K} \alpha_k(\beta, p) \theta_k(X), \\ 
&\mathbb E \left\{\theta(X)\right\} = p, \\ 
& \mathbb E \left\{\pi(X)u(X)\right\} \geq c. 
\end{align*}

As in \citet{wilder2025learning}, we reformulate the above optimization problem in terms of the dual to yield a separable inner objective:

\begin{align*}
\max_{\lambda \geq 0} \mathbb E \left\{\min_{\theta(X)\to \Delta_k} \frac{1}{\pi(X; \theta)} + \frac{1}{1 - \pi(X; \theta)}+ \sum_{j  = 1}^{J} \lambda_j(g(\theta(X), X) - c_j)\right\},
\end{align*}
where $\lambda_j$ is the dual variable associated with constraint $j$. 
The program now may be solved using the sample analogue
\begin{align*}
\max_{\lambda \geq 0} \frac{1}{n} \sum_{i  = 1}^{n} \left\{\min_{\theta(X_i ) \to \Delta_k} \frac{1}{\pi(X_i; \theta)} + \frac{1}{1 - \pi(X_i; \theta)}  + \sum_{j  = 1}^{J} \lambda_j(g(\theta(X_i), X_i) - c_j)\right\}.
\end{align*}
As all constraints are affine and the objective is convex in $\theta$, the problem is convex and so solutions may be obtained using standard solvers.

\subsection{Optimization under endogenous arrivals}\label{sec:opt_detail_endo}

Analogously, we recall the asymptotic variance under Theorem~\ref{thm:asymp_var_plm}:
\begin{align*}
\lim_{n \to \infty} n \var(\widehat \psi_Z) = \left(\mathbb E \left[\frac{\{\alpha_Q(\beta, p) - \pi(X, \theta)\}^{2}}{\sigma(X)}\right]\right)^{-1}.
\end{align*}
In the case of endogenous arrivals, $\eta = \sigma(X) = \mathbb E \left(U^{2} \mid X\right)$. 
We assume that the policymaker does not have access to this value beforehand, but as in the exogenous arrivals case that the uncertainty set is of the form $\mathcal U = \{ \eta : \|\eta \|_{\infty} \leq \delta)$.
And so, the objective in (\ref{app:eq:minmaxefficiencyobjective}) simplifies to
\begin{align*}
&\min_{\theta : \mathcal X \to \Delta_k} \max_{ \eta \in \mathcal U} \left(\mathbb E \left[\frac{\{\alpha_Q(\beta, p) - \pi(X, \theta)\}^{2}}{\eta}\right]\right)^{-1} =\min_{\theta : \mathcal X \to \Delta_k} \left(\frac{1}{\delta}\mathbb E \left[\{\alpha_Q(\beta, p) - \pi(X, \theta)\}^{2}\right]\right)^{-1} .
\end{align*}
Because $\delta$ is constant, the optimization problem under endogenous arrivals is thus
\begin{align*}
&\min_{\theta: \mathcal X \to \Delta_k} \frac{1}{\E[\{ \alpha_Q(\beta, p) - \pi(X, \theta)\}^{2}]} \\
&\pi(X; \theta) = \sum_{k = 1}^{K} \alpha_k(\beta, p) \theta_k(X), \\ 
&\mathbb E \left[\theta(X)\right] = p, \\ 
& \mathbb E \left[\pi(X)u(X)\right] \geq c. 
\end{align*}
Unlike the exogenous arrivals case, the objective is not separable when formulated as a minimization problem. 
Notice that we may consider an equivalent objective of the form
\begin{align*}
\max_{\theta : \mathcal X \to \Delta_k} \mathbb E \left[\{\alpha_Q(\beta, p) - \pi(X; \theta)\}^{2}\right],
\end{align*}
which may be rewritten as
\begin{align*}
\max_{\theta : \mathcal X \to \Delta_k} \mathbb E \left[\left\{\sum_{k = 1}^{K}\alpha_k(\beta, p)^{2} \theta_k(X) - \left(\sum_{k = 1}^{K} \alpha_k(\beta, p) \theta_k(X)\right)^{2}\right\}\right].
\end{align*}
The objective is thus to maximize the expected conditional variance induced by the randomization policy $\theta$, which is concave in $\theta$.
We may now rewrite the objective in terms of the dual
\begin{align*}
\min_{\lambda \geq 0} \max_{\theta : \mathcal X \to \Delta_k} \mathbb E \left[\left\{\sum_{k = 1}^{K}\alpha_k(\beta, p)^{2} \theta_k(X) - \left(\sum_{k = 1}^{K} \alpha_k(\beta, p)\theta_k(X)\right)^{2}\right\} - \sum^{J}_{j = 1} \lambda_j(g(\theta(X), X) - c_j)\right],
\end{align*}
where we once again use $\lambda_j$ to denote the dual variable associated with constraint $j$.
This is now separable, and so we can solve the problem using the sample population  
\begin{align*}
\min_{\lambda \geq 0} \frac{1}{n} \sum^{n}_{i=1}\max_{\theta(X_i) \to \Delta_k} \left\{\sum_{k = 1}^{K}\alpha_k(\beta, p)^{2} \theta_k(X_i) - \left(\sum_{k = 1}^{K} \alpha_k(\beta, p)\theta_k(X_i)\right)^{2}\right\} - \sum^{J}_{j = 1} \lambda_j(g(\theta(X_i), X_i) - c_j).
\end{align*}
Solutions may be obtained using standard solvers as in the exogenous case.

\subsection{Regularization}

\begin{figure}[]
    \centering
    \includegraphics[width=\textwidth]{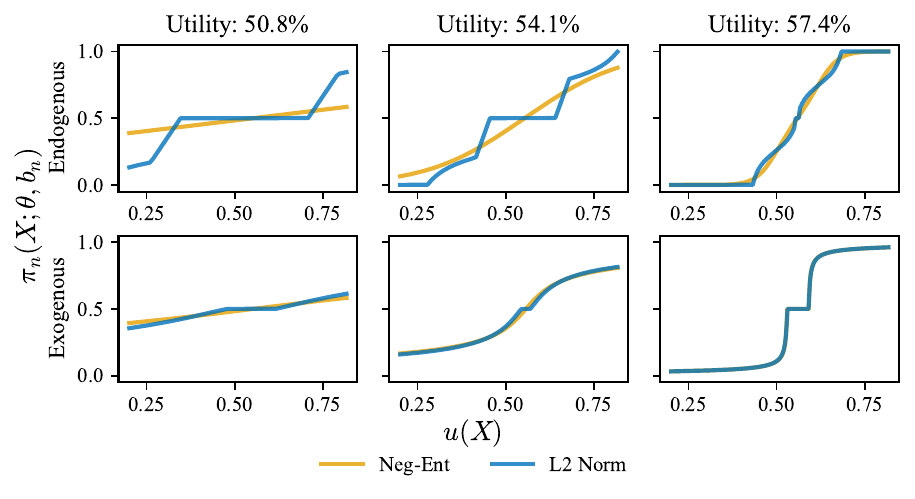}
    \caption{Effect of Negative-Entropy (orange) vs. L2 (blue) regularization on treatment probabilities $\pi_n(X; \theta, b_n)$ by individual utility $u(X)$ across varying utility constraints. L2-regularization consistently favors more deterministic solutions, assigning low treatment probability to low-utility individuals and high to high-utility, with the difference most pronounced in the endogenous setting at low levels of the utility constraint.}
    \label{app:fig:regularization}
\end{figure}

As $\pi$ is affine in $\theta$, neither objective is not strictly convex and thus there may be multiple valid solutions of $\theta$.  
To obtain a unique solution, we augment the objective with a strictly convex regularizer $R(\theta)$ that selects a single optimizer among potentially many designs that obtain the same value of the primary efficiency objective.
The choice of $R(\theta)$ can additionally encode a secondary intent; for example, a negative-entropy regularizer $R(\theta) = - \mathbb E \left[\sum_{k=1}^{K} \theta_k(X) \log(\theta_k(X))\right]$ favors higher-randomization queue assignments, while a quadratic regularizer such as $R(\theta) = \mathbb E \left[\| \theta(X) - p \|^{2}_{2}\right]$ promotes stability by discouraging deviations from the baseline proportions $p$. 
We demonstrate how the choice of regularization affects individual treatment probabilities in Figure~\ref{app:fig:regularization}.

\section{How should experiments be analyzed?}\label{sec:analysis}
As alluded in Section \ref{sec:design}, once the experiment is complete and the outcomes are observed, one can conduct standard statistical inference by constructing the Wald confidence interval $\widehat\psi \pm 1.96\sqrt{\widehat{\var}(\widehat\psi)}$ for $\psi \in \{\psi_{\ATE}, \psi_Z\}$ based on the following asymptotic normality results. Here $\widehat{\var}(\widehat\psi)$ denotes any consistent estimator of the asymptotic variance, such as the empirical variance of the estimated EIF divided by $n$.

\begin{theorem}[Asymptotic normality of DR estimator]\label{thm:asymp_normal_DR}
Assume the conditions of Theorem \ref{thm:asymp_var_DR} and additionally that there exists $m' < \infty$ such that $\E(|Y_1|^4 + |Y^0|^4) < m'$. Then conditional on $\mathcal{D}$,
\[
\sqrt{n}\big(\widehat\psi_{\ATE}^{\DR} - \psi_{\ATE}\big) \rightsquigarrow \cN\left(0, V(\theta)\right),
\]
where
\[
V(\theta) \equiv  \E\bigg\{\frac{\var(Y^1\mid X_i)}{\pi(X;\theta)} 
+ \frac{\var(Y^0\mid X_i)}{1-\pi(X;\theta)}\bigg\}
+ \var\big(\mu_1(X)-\mu_0(X)\big).
\]
\end{theorem}


\begin{corollary}[Asymptotic normality of $\widehat{\psi}_Z^\star$]\label{cor:asymp_normal_psi_Z}
Assume the conditions of Theorem \ref{thm:asymp_var_plm} and additionally that there exists $m' < \infty$ such that $\E(U^4) < m'$. Then conditional on $\mathcal{D}$,
\[
    \sqrt{n}\bigl(\widehat{\psi}_Z^\star -\psi_Z\bigr) \rightsquigarrow \cN(0,V_Z(\theta)), \qquad V_Z(\theta) \equiv \left(\E\left[\frac{\{\alpha_Q(\beta,p)-\pi(X;\theta)\}^2}{\sigma(X)}\right]\right)^{-1}.
\]
\end{corollary}

\section{Proofs}\label{sec:proofs}
In this section, we provide proofs for the results discussed in both the main text and the appendix. Throughout the proofs, for a generic random variable $V$, we use the notation $V^n \equiv (V_1,\ldots,V_n)$ and adopt shorthands $\pi_n(x) = \pi_n(x;\theta,b_n)$ and $\tilde\pi_n(x,q) = \tilde \pi_n(x,q;\theta,b_n)$ for notational convenience. We also implicitly condition on the separate independent sample $\mathcal D$ in the proofs of Theorems \ref{thm:asymp_var_DR}, \ref{thm:asymp_var_plm}, \ref{thm:asymp_normal_DR}, and Corollary \ref{cor:asymp_normal_psi_Z}.

While we suppressed unit indices throughout the main text since the data are iid, we introduce indices such as $i$ and $j$ in the proofs when needed to distinguish units or to analyze dependence induced by the queueing process.

\subsection{Proof of Theorem \ref{thm:identification_exo}}
\begin{proof}
    Fix $i \in \{1,\ldots,n\}$ and let $O_{-i} = \{(X_j, A_j, Q_j): j \neq i\}$. Then $Z_i = \Phi(A_i, Q_i, O_{-i}, b_n)$ for some measurable function $\Phi$ that encodes our design.

    By the iid assumption, together with independent queue assignment across units, $O_{-i}$ is independent of $(X_i,A_i,Q_i,Y_i^0,Y_i^1)$. Therefore
    \[
        O_{-i} \independent (A_i,Q_i,Y_i^0,Y_i^1)\mid X_i.
    \]
    Further, $Q_i \independent (A_i,Y_i^0,Y_i^1)\mid X_i$ by the queue assignment mechanism and $A_i \independent (Y_i^0,Y_i^1)\mid X_i$ by assumption. Hence
    \[
        (A_i,Q_i)\independent (Y_i^0,Y_i^1)\mid X_i.
    \]
    It follows that
    \[
        (A_i,Q_i,O_{-i}) \independent (Y_i^0,Y_i^1)\mid X_i.
    \]

    Hence for any bounded measurable $h$, we have
    \begin{align*}
        \E\{h(Z_i)\mid X_i,Y_i^0,Y_i^1\}
        &= \E\{h(\Phi(A_i,Q_i,O_{-i}, b_n))\mid X_i,Y_i^0,Y_i^1\}\\
        &= \E\{h(\Phi(A_i,Q_i,O_{-i}, b_n))\mid X_i\}\\
        &= \E\{h(Z_i)\mid X_i\},
    \end{align*}
    which is equivalent to
    \[
        Z_i \independent (Y_i^0,Y_i^1)\mid X_i.
    \]
    Since the choice of $i$ was arbitrary, this proves conditional randomization.

    Given conditional randomization, identification is straightforward. Fix any $n\geq n_0$. As an example, consider the ATE:
    \begin{align*}
        \E(Y_i^1 - Y_i^0) &= \E\{\E(Y_i^1 \mid X_i) - \E(Y_i^0 \mid X_i)\} \\
        &= \E\{\E(Y_i^1 \mid Z_i=1, X_i) - \E(Y_i^0 \mid Z_i=0, X_i)\} \\
        &= \E\{\E(Y_i \mid Z_i=1, X_i) - \E(Y_i \mid Z_i=0, X_i)\},
    \end{align*}
    where the second equality follows from conditional randomization and the third from consistency. Positivity allows conditional means to be well-defined under the $n$-unit design.
\end{proof}

\subsection{Proof of Lemma \ref{lem:valid_iv}}
\begin{proof}
Let $Z^k$ denote the potential treatment status if we set $Q=k$. Further, let $Y^{zk}$ denote the potential outcome if we set $Z=z$ and $Q=k$. 

First consider conditional IV-unconfoundedness. Since queue assignments are conditionally randomized given $X$ and $r_n(X,Q)$ is a function only of $(X,Q)$, it follows that
\[
    r_n(X,Q) \independent \{(Y^{0k}, Y^{1k}, Z^k): k=1,\ldots,K\} \mid X.
\]

Next consider exclusion. By design, queue assignment only determines priority order for treatments and does not alter any other program feature. Therefore, under our stationary potential outcome formulation in which outcomes depend on whether treatment is eventually received by the horizon, but not on the precise timing of receipt, queue assignment affects outcomes only through terminal treatment status:
\[
    Y^{zk} = Y^z
    \quad \text{for all } z\in\{0,1\},\ k\in\{1,\ldots,K\}.
\]
Then any function of the queue assignment, such as $r_n(X,Q)$, can affect the observed outcome only through $Z$. Hence $r_n(X,Q)$ also satisfies exclusion.

Finally, note that monotonicity based on queue assignments
\[
    Z^k \geq Z^\ell \quad \text{for all } k < \ell,
\]
also follows straightforwardly from our design; moving a unit to a higher-priority queue cannot reduce its treatment status. Hence taking conditional expectations yields
\[
    \tilde\pi_n(x,k) = \E(Z^k \mid X=x) \geq \E(Z^\ell \mid X=x) = \tilde\pi_n(x,\ell) \quad \text{for all } k < \ell. \qedhere
\]
\end{proof}

\subsection{Proof of Theorem \ref{thm:arrival_endo_late}}
\begin{proof}
First, note that we can write
\[
    Y = Y^0 + Z^{Q}(Y^1 - Y^0).
\]
Consider $\E\{r_n(X,Q)Y \mid X=x\}$. Substituting the above yields 
\begin{align*}
    \E\{&r_n(X,Q)Y \mid X=x\} = \E\Big[(r_n(X,Q)\{Y^0 + Z^{Q}(Y^1 - Y^0)\} \mid X=x\Big] \\  
    &= \E\{r_n(X,Q)Y^0 \mid X=x\} + \E\Big[(r_n(X,Q)Z^{Q}(Y^1 - Y^0) \mid X=x\Big] \\
    &= 0 + \sum_{k=1}^K \theta_k(x)\{\tilde\pi_n(x,k) - \pi_n(x)\}\E\{Z^k(Y^1 - Y^0) \mid X=x\} \\
    &= \sum_{k < \ell}\theta_k(x)\theta_\ell(x)\{\tilde\pi_n(x,k) - \tilde\pi_n(x,\ell)\}\E\{(Z^k -Z^\ell)(Y^1 - Y^0) \mid X=x\},
\end{align*}
where the second equality holds from $\E\{r_n(X, Q) \mid X\} = 0$ and conditional unconfoundeness.

Now recall that we have $Z^k \geq Z^\ell$ for $k < \ell$ by monotonicity, which implies $Z^k - Z^\ell = \one(i \in \mathcal C_{k,\ell})$. It follows that
\begin{align*}
    \E\{(Z^k -Z^\ell)(Y^1 - Y^0) \mid X=x\} &= \bP(i \in \mathcal C_{k,\ell} \mid X=x)\E(Y^1 - Y^0 \mid X=x, i \in \mathcal C_{k,\ell}) \\
    &= \{\tilde\pi_n(x,k) - \tilde\pi_n(x,\ell)\}\psi_{k,\ell}(x).
\end{align*}
Hence 
\[
    \E\{r_n(X,Q)Y \mid X=x\} = \sum_{k < \ell}\theta_k(x)\theta_\ell(x)\{\tilde\pi_n(x,k) - \tilde\pi_n(x,\ell)\}^2\psi_{k,\ell}(x),
\]
and subsequently
\[
\E\{r_n(X,Q)Y\} =  \E\left\{\sum_{k < \ell}w_{k,\ell}(X)\psi_{k,\ell}(X)\right\}.
\]
Now consider $\E\{r_n(X,Q)Z \mid X=x\}$. Since $Z = Z^{Q}$, an analogous derivation yields
\begin{align*}
    \E\{r_n(X,Q)Z \mid X=x\} &= \sum_{k=1}^K \theta_k(x)\{\tilde\pi_n(x,k) - \pi_n(x)\}\E\{Z^k \mid X=x\} \\
    &= \sum_{k < \ell}\theta_k(x)\theta_\ell(x)\{\tilde\pi_n(x,k) - \tilde\pi_n(x,\ell)\}^2.
\end{align*}
Hence $\E\{r_n(X,Q)Z\} = \E\{\sum_{k < \ell}w_{k,\ell}(X)\}$, and we have the desired.
\end{proof}

\subsection{Proof of Theorem \ref{thm:pliv_instrument}}
\begin{proof}
    This is a direct consequence of Theorem \ref{thm:asymp_var_plm}.
\end{proof}

\subsection{Proof of Theorem \ref{thm:propensity}}
To show Theorem \ref{thm:propensity} holds, we first prove a lemma that shows the number of the eventually treated units in queues $1,\dots,k$ converges to $\min\{\beta,c_k\}$ defined in the statement of Theorem \ref{thm:propensity} under Assumptions \ref{assm:fixed_queue}--\ref{assm:arrivals_indep_covariates}, which is an intuitive result under our design.
\begin{lemma}\label{lem:eventual_treated_mass}
Suppose Assumptions \ref{assm:fixed_queue}--\ref{assm:arrivals_indep_covariates} hold. Let $Z_{i,t}$ be the treatment status of unit $i$ at time $t$ where $Z_{i,\tau} = Z_i$. For each queue $k \in \{1,\ldots,K\}$ and each treatment allocation time $t \in \{1,\ldots,\tau\}$, define
\[
    \bar T_{\leq k,t,n} \equiv \frac{1}{n}\sum_{i=1}^n \one(Q_i \le k, Z_{i,t} = 1).
\]
Then, as $n \to \infty$,
\[
    \sup_{1 \le t \le \tau}\left|\bar T_{\leq k,t,n} - \min\{\beta,c_k\}\bP(A_i \leq t)\right| \to 0 \quad \text{a.s.}
\]
In particular,
\[
    \frac{T_{\le k,n}}{n} \to \min\{\beta,c_k\} \quad \text{a.s.}, \quad T_{\le k,n} \equiv \sum_{i=1}^n \one(Q_i \le k, Z_i = 1).
\] 
\end{lemma}
\begin{proof}
For each $k \in \{1,\dots,K\}$ and $t \in \{1,\ldots,\tau\}$, define
\begin{align*}
    \bar N_{\leq k, \leq t, n} &\equiv \frac{1}{n}\sum_{i=1}^n \one(Q_i \le k, A_i \leq t), \\
    \bar W_{\leq k, \leq t, n} &\equiv \frac{1}{n}\sum_{i=1}^n \one(Q_i \le k, A_i \leq t, Z_{i,t} = 0), \\
    \bar b_{\leq t, n} &\equiv \frac{1}{n}\sum_{s=1}^t b_{s,n}, 
\end{align*}
which denotes, on $1/n$ scale, the average number of units that arrive in queues up to queue $k$ by period-$t$ allocation, the average number of units that arrive in queues up to queue $k$ by period-$t$ allocation and have not been treated, and the average budget up to period-$t$ allocation, respectively, with the conventions
\[
    \bar N_{\le k,\le 0,n} = 0, \quad \bar b_{\le 0,n} = 0, \quad \bar W_{\le k,\le 0,n} = 0.
\]
Further, define 
\[
    \Delta \bar N_{\leq k,t,n} \equiv \bar N_{\le k, \le t, n} - \bar N_{\le k, \le t-1, n}, 
\]
which corresponds to the units that arrive after period-$t-1$ allocation and before period-$t$ allocation. Our first goal is to show
\begin{equation}\label{eq:avg_waiting_mass}
\bar W_{\le k, \le t, n} = \max_{0\le s \le t}\Bigl\{\bigl(\bar N_{\le k, \le t, n} - \bar N_{\le k, \le s, n}\bigr) - \bigl(\bar b_{\le t, n} - \bar b_{\le s, n}\bigr)
\Bigr\},    
\end{equation}
for every $t$. Fix $k$ and consider the quantity
\[
    \bar W_{\le k,\le t-1,n} + \Delta \bar N_{\le k, t, n},
\]
which corresponds to the average number of waiting units in queues up to queue $k$ \emph{immediately} before the period-$t$ allocation. 
Since the average number of period-$t$ treatment slots used on queues $1,\dots,k$ is exactly
\[
    \min\left\{\bar W_{\le k,\le t-1,n} + \Delta \bar N_{\leq k,t,n}, b_{t,n}/n\right\},
\]  
by our design, we can write
\begin{align*}
    \bar W_{\le k, \le t, n} &= \bar W_{\le k,\le t-1,n} + \Delta \bar N_{\leq k,t,n} - \min\{\bar W_{\le k,\le t-1,n} + \Delta \bar N_{\leq k,t,n},b_{t,n}/n\} \\
    &= \bigl(\bar W_{\le k,\le t-1,n} + \Delta \bar N_{\leq k,t,n} - b_{t,n}/n\bigr)_+.
\end{align*}

Next we use induction to prove \eqref{eq:avg_waiting_mass}. For the base case $t=0$, both sides are zero. Now suppose \eqref{eq:avg_waiting_mass} holds for $t-1$. Then
\begin{align*}
\bar W_{\le k, \le t, n}
&= \biggl( \max_{0\le s \le t-1} \Bigl\{ \bigl(\bar N_{\le k,\le t-1,n} - \bar N_{\le k,\le s, n}\bigr) - \bigl(\bar b_{\le t-1,n} - \bar b_{\le s,n}\bigr)\Bigr\}+\Delta \bar N_{\leq k,t,n} - b_{t,n}/n \biggr)_+ \\
&= \biggl( \max_{0\le s \le t-1} \bigl\{ \bigl(\bar N_{\le k,\le t,n} - \bar N_{\le k, \le s, n}\bigr) - \bigl(\bar b_{\le t, n} - \bar b_{\le s,n}\bigr) \bigr\}\biggr)_+ \\
&= \max_{0\le s \le t} \Bigl\{ \bigl(\bar N_{\le k,\le t,n} - \bar N_{\le k, \le s, n}\bigr) - \bigl(\bar b_{\le t, n} - \bar b_{\le s,n}\bigr) \Bigr\},
\end{align*} 
where the second equality follows from $\bar N_{\le k,\le t-1,n} + \Delta \bar N_{\leq k,t,n} = \bar N_{\le k,\le t,n}$ and $\bar b_{\le t-1,n} + b_{t,n}/n = \bar b_{\le t, n}$, and the third from the fact that the term corresponding to $s=t$ equals $0$.

Now let $a_t \equiv \bP(A_i \leq t)$. As a second step, we show
\[
\sup_{1\le t\le \tau} \Bigl| \bar W_{\le k, \le t, n} - \max_{0\le s \le t}(c_k-\beta)(a_t-a_s) \Bigr| \to 0,
\]
and 
\[
    \max_{0 \leq s \leq t}(c_k-\beta)(a_t-a_s) = (c_k-\beta)_+ a_t,
\]
to arrive at 
\[
    \sup_{1\leq t\leq \tau} \bigl| \bar W_{\le k, \le t,n} - (c_k-\beta)_+ a_t \bigr| \to 0.
\]

Recall that Assumption \ref{assm:stable_budget} states
\[
    \sup_{1\leq t\leq \tau} \bigl|\bar b_{\le t, n} - \beta a_t\bigr| \to 0,
\]
and Assumptions \ref{assm:fixed_queue} and \ref{assm:arrivals_indep_covariates} along with strong law of large numbers give
\[
    \bar N_{\leq k, \leq t, n} = \frac{1}{n}\sum_{i=1}^n \one(Q_i \le k, A_i \leq t) \to c_ka_t \quad \text{a.s.}
\]
As a result,
\[
    \sup_{1\le t\le \tau}\bigl|\bar N_{\leq k, \le t,n} - c_k a_t\bigr| + \sup_{1\le t\le \tau}\bigl|\bar b_{\le t, n} - \beta a_t\bigr| \to 0 \quad \text{a.s.}
\]

Now for every $0\leq s \leq t \leq \tau$, note that
\begin{align*}
&
\Bigl|\bigl(\bar N_{\le k, \le t,n} - \bar N_{\le k, \le s, n}\bigr)-\bigl(\bar b_{\le t, n} - \bar b_{\le s,n}\bigr) - (c_k-\beta)(a_t-a_s) \Bigr|
\\
&\leq \bigl| \bar N_{\le k, \le t,n} - c_k a_t \bigr| + \bigl| \bar N_{\le k, \le s, n} - c_k a_s \bigr|
+ \bigl| \bar b_{\le t, n} - \beta a_t \bigr| + \bigl| \bar b_{\le s,n} - \beta a_s \bigr| \\
& \leq 2\bigg\{\sup_{1\le t\le \tau}\bigl|\bar N_{\leq k, \le t,n} - c_k a_t\bigr| + \sup_{1\le t\le \tau}\bigl|\bar b_{\le t, n} - \beta a_t\bigr|\bigg\}.
\end{align*}
Taking the maximum over $0\le s \le t$ gives
\[
\sup_{1\le t\le \tau} \Bigl| \bar W_{\le k, \le t, n} - \max_{0\le s \le t}(c_k-\beta)(a_t-a_s) \Bigr| \to 0.
\]

Since $a_t$ is non-decreasing in $t$ and $a_0=0$, there are two cases: 
\begin{enumerate}
    \item If $c_k \ge \beta$, then $(c_k-\beta)(a_t-a_s)$ is maximized at $s=0$, and
        \[
        \max_{0\le s \le t}(c_k-\beta)(a_t-a_s) = (c_k-\beta)a_t.
        \]
    \item If $c_k < \beta$, then $(c_k-\beta)(a_t-a_s)\le 0$ for all $s\le t$, so
        \[
        \max_{0\le s \le t}(c_k-\beta)(a_t-a_s) = 0.
        \]
\end{enumerate}
It follows that
\[
    \sup_{1 \leq t \leq \tau} \bigl| \bar W_{\le k, \le t,n} - (c_k-\beta)_+ a_t \bigr| \to 0.
\]

Finally, since
\[
    \bar T_{\le k, t,n} = \bar N_{\le k, \le t,n} - \bar W_{\le k, \le t, n},
\]
by construction, we have
\begin{align*}
\sup_{1\leq t\leq \tau} \bigl| \bar T_{\le k, t,n} &- \min\{\beta,c_k\}a_t \bigr| = \sup_{1 \leq t \leq \tau} \Bigl| \bar T_{\le k, t,n} - \bigl\{ c_k a_t - (c_k-\beta)_+ a_t \bigr\} \Bigr| \\
&= \sup_{1 \leq t \leq \tau} \Bigl| \bar N_{\le k, \le t,n} - \bar W_{\le k, \le t, n} - \bigl\{ c_k a_t - (c_k-\beta)_+ a_t \bigr\} \Bigr| \\
&\leq \sup_{1 \leq t \leq \tau} \bigl|\bar N_{\le k, \le t, n} - c_k a_t\bigr| + \sup_{1 \leq t \leq \tau} \bigl| \bar W_{\le k,\le t, n} - (c_k-\beta)_+ a_t \bigr|  \to 0.    
\end{align*}

Taking $t=\tau$ and using $a_{\tau}=1$ gives $T_{\le k,n}/n = \bar T_{\le k,\tau,n} \to \min\{\beta,c_k\}$. 
\end{proof}

\begin{proof}[Proof of Theorem \ref{thm:propensity}]
all statements involving conditioning on $X_i = x$ should be understood to hold for almost every $x$, and all convergences below are almost sure convergences except for the deterministic budget path.

It suffices to show that, for each $k$ with $p_k > 0$ and each $x$ such that
$\theta_k(x) > 0$,
\[
\tilde\pi_n(x,k) = \bP_{\theta,b_n}(Z_i=1\mid X_i=x,Q_i=k) \to \alpha_k(\beta,p).
\]
Fix such a $k$ and consider
\[
T_{k,n} \equiv \sum_{i=1}^n \one(Q_i = k, Z_i = 1),
\quad
N_{k,n} \equiv \sum_{i=1}^n \one(Q_i = k),
\]
which represent the number of individuals in queue $k$ who are treated by the end of the allocation process, and the total number of individuals that end up in queue $k$, respectively.

Let $J_k \equiv \{j : Q_j = k\}$ where $|J_k| = N_{k,n}$. Conditional on $Q^n = q^n$ with $q_i = k$, the set $J_k$ is fixed. Since the arrival times are iid, continuous, and independent of $(X_i,Q^n)$, the arrival order within $J_k$ is a uniform random permutation even after conditioning on $X_i = x$. Under our treatment allocation rule, for any realization of the arrivals, exactly $T_{k,n}$ of the $N_{k,n}$ units in queue $k$ are treated by the end of the allocation process. Therefore, we have
\[
\tilde\pi_n(x,k) = 
\E\left(\frac{T_{k,n}}{N_{k,n}} \biggm| X_i = x, Q_i = k\right).
\]

Now consider
\[
T_{\leq k,n} \equiv \sum_{i=1}^n \one(Q_i \leq k, Z_i = 1), \quad T_{\leq 0,n} \equiv 0,
\] 
as defined in Lemma \ref{lem:eventual_treated_mass} and note that $T_{k,n} = T_{\leq k,n} - T_{\leq k-1,n}$. It follows by Lemma \ref{lem:eventual_treated_mass} that
\begin{align*}
\frac{T_{k,n}}{n} &= \frac{T_{\leq k,n} - T_{\leq k-1,n}}{n} \to \min\{\beta,c_k\} - \min\{\beta,c_{k-1}\} \\
&= \beta-(\beta-c_k)_+ - \{\beta-(\beta-c_{k-1})_+\} \\
&= (\beta-c_{k-1})_+ - (\beta-c_k)_+,
\end{align*}
and
\[
\frac{T_{k,n}}{N_{k,n}} = \frac{T_{k,n}/n}{N_{k,n}/n} \to \frac{(\beta-c_{k-1})_+ - (\beta-c_k)_+}{p_k} = \alpha_k(\beta,p).
\]
Therefore, dominated convergence implies
\[
\E\left(\left.\frac{T_{k,n}}{N_{k,n}}\right| X_i = x, Q_i = k\right) \to \alpha_k(\beta,p). \qedhere
\]
\end{proof}

\subsection{Proof of Theorem \ref{thm:asymp_var_DR}}
We first present two lemmas that are necessary to show Theorem \ref{thm:asymp_var_DR}. The first lemma says that changing one unit's queue assignment or arrival time can change the number of eventually treated units in any queue-period cell by at most a constant. The second lemma is a more technical one that shows the quantity
\[
    \Big|\E\Big[\{Z_i-\pi_n(X_i)\}\{Z_j-\pi_n(X_j)\}\mid X_i,X_j\Big]\Big|,
\]
is of order $n^{-1}$.

Throughout the proofs of the two lemmas, we write $S_i=s$ to denote $A_i\in [s-1,s)$ and $C_i=(Q_i,S_i) \in \mathcal C=\{1,\ldots,K\}\times\{1,\ldots,\tau\}$.

\begin{lemma}\label{lem:bounded_treatment_count}
Define $V \equiv \{(Q_1,A_1),\ldots,(Q_n,A_n)\}$. For each $r\in\{1,\ldots,n\}$, let $V^{(r)}$ be the vector obtained from $V$ by replacing its $r$-th coordinate $(Q_r,A_r)$ by an arbitrary $(Q_r',A_r')$. For $c=(k,s)\in\mathcal C$, define
\[
    T_{c,n}(V) \equiv \sum_{i=1}^n \one(C_i=c, \ Z_i(V)=1).
\]
Then there exists a constant $L=L(K,\tau)<\infty$ such that, for every
$r\in\{1,\ldots,n\}$,
\[
    \sup_{c\in\mathcal C}\left|T_{c,n}(V)-T_{c,n}(V^{(r)})\right| \le L.
\]
\end{lemma}

\begin{proof}
For $c=(k,s)\in\mathcal C$ and $t\in\{1,\ldots,\tau\}$, define
\begin{align*}
    W_{c,\le t,n}(V) &\equiv \sum_{i=1}^n \one(C_i=c,\ Z_{i,t-1}(V)=0,\ s\le t), \\
    D_{c,t,n}(V) &\equiv \sum_{i=1}^n \one(C_i=c,\ Z_{i,t-1}(V)=0,\ Z_{i,t}(V)=1),
\end{align*}
where $W_{c,\le t,n}(V)$ is the number of untreated units from cell $c=(k,s)$ available for treatment at the beginning of period $t$, and $D_{c,t,n}(V)$ is the number newly treated from cell $c$ in period $t$. By construction,
\[
    T_{c,n}(V)=\sum_{t=s}^{\tau}D_{c,t,n}(V).
\]
By our treatment rule, cells are served in lexicographic order: first by queue
$k$, then by arrival period $s$. Hence for $s>t$, $D_{c,t,n}(V)=0$. For $c=(k,s)$ with $s\le t$, define
\[
    R^-_{c,t,n}(V) \equiv b_{t,n} - \sum_{\ell<k}\sum_{u\le t}W_{(\ell,u),\le t,n}(V) - \sum_{u<s}W_{(k,u),\le t,n}(V),
\]
and
\[
    R^+_{c,t,n}(V) \equiv b_{t,n} - \sum_{\ell<k}\sum_{u\le t}W_{(\ell,u),\le t,n}(V) - \sum_{u\le s}W_{(k,u),\le t,n}(V),
\]
as the the remaining budget after serving all cells that have higher priority than $(k,s)$, and the remaining budget after \emph{also} serving cell $(k,s)$, respectively. Then
\[
    D_{c,t,n}(V)=\{R^-_{c,t,n}(V)\}_+ - \{R^+_{c,t,n}(V)\}_+.
\]

Now let
\[
    \Delta_t \equiv \sum_{c\in\mathcal C}\left|W_{c,\le t,n}(V)-W_{c,\le t,n}(V^{(r)})\right|,
\]
which measures the total discrepancy in the number of units who have arrived by period $t$ and have not yet been treated under $V$ versus $V^{(r)}$. We use $\Delta_t$ to upper bound $|D_{c,t,n}(V)-D_{c,t,n}(V^{(r)})|$, and show that $\Delta_t$ is bounded by some finite constant $C(K,\tau)$.

Suppose $t=1$. Changing one coordinate of $V$ can remove at most one unit from one cell and add at most one unit to another, so $\Delta_1\le 2$. 

Next, fix $t\in\{1,\ldots,\tau-1\}$. If $s>t$, then $D_{c,t,n}(V)=D_{c,t,n}(V^{(r)})=0$. If $s\le t$, then 
\[
\begin{aligned}
&\left|D_{c,t,n}(V)-D_{c,t,n}(V^{(r)})\right| \\
&=
\left|\{R^-_{c,t,n}(V)\}_+ - \{R^+_{c,t,n}(V)\}_+ - \{R^-_{c,t,n}(V^{(r)})\}_+ + \{R^+_{c,t,n}(V^{(r)})\}_+\right| \\
&\leq \left|\{R^-_{c,t,n}(V)\}_+ - \{R^-_{c,t,n}(V^{(r)})\}_+\right| + \left|\{R^+_{c,t,n}(V)\}_+ - \{R^+_{c,t,n}(V^{(r)})\}_+\right| \\
&\leq \left|R^-_{c,t,n}(V)-R^-_{c,t,n}(V^{(r)})\right| + \left| R^+_{c,t,n}(V)-R^+_{c,t,n}(V^{(r)}) \right|.
\end{aligned}
\]
where the last line follows from $u\mapsto u_+$ being $1$-Lipschitz. Further note that
\[
\left|
    R^-_{c,t,n}(V)-R^-_{c,t,n}(V^{(r)})
\right|
\le
\sum_{c'\in\mathcal C}
\left|
    W_{c',\le t,n}(V)-W_{c', \le t,n}(V^{(r)})
\right|
=
\Delta_t,
\]
and similarly $\left|R^+_{c,t,n}(V)-R^+_{c,t,n}(V^{(r)})\right| \le \Delta_t$. Therefore, for every $c\in\mathcal C$,
\[
\left|D_{c,t,n}(V)-D_{c,t,n}(V^{(r)})\right|
\le 2\Delta_t .
\]
Summing over $c\in\mathcal C$ gives
\[
    \sum_{c\in\mathcal C}\left|D_{c,t,n}(V)-D_{c,t,n}(V^{(r)})\right|
    \leq \sum_{c\in\mathcal C}2\Delta_t = 2K\tau\,\Delta_t .
\]

Now let $A^{\mathrm{new}}_{c,t+1,n}(V)$ denote the number of units from cell $c$ that become newly available in period $t+1$. Then we can write
\[
    W_{c,\le t+1,n}(V)
    =
    W_{c,\le t,n}(V)-D_{c,t,n}(V)+A^{\mathrm{new}}_{c,t+1,n}(V).
\]
Now triangle inequality implies
\[
\begin{aligned}
&\Delta_{t+1} = \sum_{c\in\mathcal C}\left|W_{c,\le t+1,n}(V)-W_{c,\le t+1,n}(V^{(r)})\right| \\
&\leq
\sum_{c\in\mathcal C}\left|W_{c,\le t,n}(V)-W_{c,\le t,n}(V^{(r)})\right| + \sum_{c\in\mathcal C}\left|D_{c,t,n}(V)-D_{c,t,n}(V^{(r)})\right| + \sum_{c\in\mathcal C}\left|A^{\mathrm{new}}_{c,t+1,n}(V)-A^{\mathrm{new}}_{c,t+1,n}(V^{(r)})\right| \\
&\leq (1+2K\tau)\Delta_t + \sum_{c\in\mathcal C} \left| A^{\mathrm{new}}_{c,t+1,n}(V)-A^{\mathrm{new}}_{c,t+1,n}(v^{(r)}) \right|.
\end{aligned}
\]
Since $V$ and $V^{(r)}$ differ in only one coordinate, the newly available arrival-count vectors in period $t+1$ differ in $\ell_1$ distance by at most $2$. Thus
\[
\Delta_{t+1} \leq \Delta_t + 2K\tau\,\Delta_t + 2 = (1+2K\tau)\Delta_t+2,
\]
and subsequently, after some algebra, $\Delta_t \leq 2\sum_{h=0}^{t-1}(1+2K\tau)^h$. Since $\tau$ is fixed, we obtain
\[
    \max_{1\le t\le \tau}\Delta_t \leq 2\sum_{h=0}^{\tau-1}(1+2K\tau)^h \equiv C(K,\tau)
\]
Consequently, $\left|D_{c,t,n}(V)-D_{c,t,n}(V^{(r)})\right| \le 2C(K,\tau)$ for every $c\in\mathcal C$ and every $t$.

Finally, for $c=(k,s)$,
\[
\begin{aligned}
\left|T_{c,n}(V)-T_{c,n}(V^{(r)})\right|
&=
\left|\sum_{t=s}^{\tau}\{D_{c,t,n}(V)-D_{c,t,n}(V^{(r)})\}\right| \\
&\leq \sum_{t=s}^{\tau}\left|D_{c,t,n}(V)-D_{c,t,n}(V^{(r)})\right| \\
&\leq 2\tau C(K,\tau).
\end{aligned}
\]
Taking the supremum over $c\in\mathcal C$, the result holds with
$L(K,\tau)=2\tau C(K,\tau)$.
\end{proof}

\begin{lemma}
\label{lem:remainder_var_multi}
Suppose Assumption \ref{assm:arrivals_indep_covariates} holds. For $c=(k,s)\in\mathcal C$, let $\lambda_{\min}\equiv\min\{\lambda_c:\lambda_c>0\}>0$ where $\lambda_c \equiv \bP(C_i=c)=p_k\bP(S_i=s)$. Then there exists $C=C(\lambda_{\min},L)<\infty$ such that, for all $n\ge2$ and all $i\neq j$,
\[
\Big|\E\Big[\{Z_i-\pi_n(X_i)\}\{Z_j-\pi_n(X_j)\}\mid X_i,X_j\Big]\Big|
\le \frac{C}{n}
\quad\textnormal{a.s.}
\]
\end{lemma}

\begin{proof}
Note that
\begin{align}
    &\E\big[
        \{Z_i-\pi_n(X_i)\}
        \{Z_j-\pi_n(X_j)\}
        \mid X_i,X_j
    \big] \nonumber \\
    &=
    \cov(Z_i,Z_j\mid X_i,X_j) \label{eq:mp_lemma_cov} \\
    &\quad+
    \{\E(Z_i\mid X_i,X_j)-\pi_n(X_i)\}
    \{\E(Z_j\mid X_i,X_j)-\pi_n(X_j)\}.
    \label{eq:mp_lemma_mean_mismatch}
\end{align}
The goal of our proof is to show that \eqref{eq:mp_lemma_cov} is $O(n^{-1})$
and \eqref{eq:mp_lemma_mean_mismatch} is $O(n^{-2})$.

First we show \eqref{eq:mp_lemma_cov} is $O(n^{-1})$. For each $c \in \mathcal C$, define
\[
    N_{c,n}\equiv \sum_{i=1}^n\one(C_i=c), \qquad T_{c,n}\equiv \sum_{i=1}^n\one(C_i=c,Z_i=1),
\]
which represent the total and the treated number of units in cell $c$. Further, define $m_c(C^n)\equiv T_{c,n}/N_{c,n}$ whenever $N_{c,n}>0$. Conditional on $X_i,X_j,C^n$, as in the proof of Theorem \ref{thm:propensity}, we have
\[
    \E(Z_i\mid X_i,X_j,C^n)=m_{C_i}(C^n),
    \qquad
    \E(Z_j\mid X_i,X_j,C^n)=m_{C_j}(C^n).
\]
So by the law of total covariance,
\begin{align}
&\cov(Z_i,Z_j\mid X_i,X_j) \nonumber \\
&= \E\{\cov(Z_i,Z_j\mid X_i,X_j,C^n)\mid X_i,X_j\} 
+ \cov\{\E(Z_i\mid X_i,X_j,C^n),\E(Z_j\mid X_i,X_j,C^n)\mid X_i,X_j\} \nonumber \\
&= \E\{\cov(Z_i,Z_j\mid X_i,X_j,C^n)\mid X_i,X_j\} + \cov\{m_{C_i}(C^n),m_{C_j}(C^n)\mid X_i,X_j\}.
\label{eq:mp_cov_decomp}
\end{align}

For the first term in \eqref{eq:mp_cov_decomp}, note that $\cov(Z_i,Z_j\mid X_i,X_j,C^n)=0$ when $C_i\neq C_j$ since, conditional on $C^n$, the remaining randomness comes from independent within-cell arrival orders. In the case $C_i=C_j=c$, sampling without replacement gives
\begin{align*}
    \left|\cov(Z_i,Z_j\mid X_i,X_j,C^n)\right|
    &=
    \left|
    \E(Z_iZ_j\mid X_i,X_j,C^n)
    -
    \E(Z_i\mid X_i,X_j,C^n)\E(Z_j\mid X_i,X_j,C^n)
    \right| \\
    &=
    \frac{(T_{c,n}/N_{c,n})(1-T_{c,n}/N_{c,n})}{N_{c,n}-1}
    \le
    \frac{1}{4(N_{c,n}-1)}.
\end{align*}
It follows that
\[
    \left|\E\{\cov(Z_i,Z_j\mid X_i,X_j,C^n)\mid X_i,X_j\}\right| \le \frac14
    \E\left\{\frac{\one(C_i=C_j)}{N_{C_i,n}-1}\biggm| X_i,X_j\right\}.
\]
Now fix $c\in\mathcal C$ with $\lambda_c>0$. Conditional on $X_i,X_j$ and $C_i=C_j=c$, we claim that the random variables $\one(C_m=c)$, $m\neq i,j$, remain iid Bernoulli with probability $\lambda_c$. To see this, note that the units are iid across $m$, and the queue assignments and arrival times are generated independently across units. Therefore, for every $m\neq i,j$, $C_m$ is independent of $(X_i,X_j,C_i,C_j)$. Hence conditioning on $X_i,X_j$ and on the event $C_i=C_j=c$ does not change the distribution of $C_m$, and
\[
    \bP(C_m=c\mid X_i,X_j,C_i=C_j=c) = \bP(C_m=c)=\lambda_c.
\]
Moreover, the variables $\{C_m:m\neq i,j\}$ remain mutually independent under this conditioning. Consequently,
\[
    N_{c,n}
    =
    2+\sum_{m\neq i,j}\one(C_m=c)
    \equiv 2+B,
    \qquad
    B\sim\textnormal{Bin}(n-2,\lambda_c).
\]
It follows that
\begin{align*}
    \E\left(\frac{1}{N_{c,n}-1}\biggm| C_i=C_j=c\right)
    &=
    \E\left(\frac{1}{B+1}\right) \\
    &=
    \sum_{b=0}^{n-2}
    \binom{n-2}{b}
    \lambda_c^b(1-\lambda_c)^{n-2-b}
    \int_0^1 t^b dt \\
    &=
    \frac{1-(1-\lambda_c)^{n-1}}{(n-1)\lambda_c}
    \le
    \frac{1}{(n-1)\lambda_c}
    \le
    \frac{1}{(n-1)\lambda_{\min}}.
\end{align*}
Therefore
\[
    \left|
    \E\{\cov(Z_i,Z_j\mid X_i,X_j,C^n)\mid X_i,X_j\}
    \right|
    \le
    \frac{1}{4(n-1)\lambda_{\min}}.
\]

Now we consider the second term in \eqref{eq:mp_cov_decomp}. Write
\begin{align}
&\cov\{m_{C_i}(C^n),m_{C_j}(C^n)\mid X_i,X_j\} \nonumber \\
&=
\E\left[
    \cov\{m_{C_i}(C^n),m_{C_j}(C^n)\mid X_i,X_j,C_i,C_j\}
    \mid X_i,X_j
\right] \label{eq:mp_cov_es1} \\
&+
\cov\left(
    \E\{m_{C_i}(C^n)\mid X_i,X_j,C_i,C_j\},
    \E\{m_{C_j}(C^n)\mid X_i,X_j,C_i,C_j\}
    \mid X_i,X_j
\right) \label{eq:mp_cov_es2}.
\end{align}

We first show that \eqref{eq:mp_cov_es1} is $O(n^{-1})$ by using Cauchy--Schwarz and applying Efron--Stein inequality. Consider the event $C_i=c$ and $C_j=d$, and define $V\equiv (C_m)_{m\neq i,j}$.Then $m_c(C^n)$ can be viewed as a function of $V$ with $C_i=c$ and $C_j=d$ fixed, and we write it as $m_c(V)$. For each $r\in\{1,\ldots,n-2\}$, let $V^{(r)}$ be the vector obtained from $V$ by replacing its $r$-th coordinate by an independent copy $C_r'$ with the same distribution as $C_r$. Changing one coordinate of $V$ changes one unit's queue-period cell, which can be realized by changing that unit's $(Q,A)$ coordinate; hence Lemma~\ref{lem:bounded_treatment_count} applies.

Then for each $c' \in\mathcal C$, we have
\[
    |N_{c',n}(V)-N_{c',n}(V^{(r)})|\le 1.
\]
Moreover, by Lemma~\ref{lem:bounded_treatment_count},
\[
    |T_{c',n}(V)-T_{c',n}(V^{(r)})|\le L.
\]
It follows that
\begin{align*}
    |m_c(V)-m_c(V^{(r)})|
    &=
    \left|
    \frac{T_{c,n}(V)}{N_{c,n}(V)}
    -
    \frac{T_{c,n}(V^{(r)})}{N_{c,n}(V^{(r)})}
    \right| \\
    &=
    \frac{
    \left|
    T_{c,n}(V)N_{c,n}(V^{(r)})
    -
    T_{c,n}(V^{(r)})N_{c,n}(V)
    \right|
    }{
    N_{c,n}(V)N_{c,n}(V^{(r)})
    } \\
    &\le
    \frac{
    |T_{c,n}(V)-T_{c,n}(V^{(r)})|N_{c,n}(V^{(r)})
    +
    T_{c,n}(V^{(r)})|N_{c,n}(V)-N_{c,n}(V^{(r)})|
    }{
    N_{c,n}(V)N_{c,n}(V^{(r)})
    } \\
    &\le
    \frac{L N_{c,n}(V^{(r)})+N_{c,n}(V^{(r)})}{
    N_{c,n}(V)N_{c,n}(V^{(r)})
    }
    =
    \frac{L+1}{N_{c,n}(V)}.
\end{align*}
Now, note that $N_{c,n}(V)\ge 1+B$ conditional on $C_i=c$ and $C_j=d$ where $B\sim\textnormal{Bin}(n-2,\lambda_c)$ as before. Therefore
\begin{align*}
    \E\left[
        \{m_c(V)-m_c(V^{(r)})\}^2
        \mid X_i,X_j,C_i=c,C_j=d
    \right] &\le
    (L+1)^2\E\left(\frac{1}{(B+1)^2}\right) \\
    &\le
    2(L+1)^2\E\left(\frac{1}{(B+1)(B+2)}\right).
\end{align*}
Now using $1/\{(b+1)(b+2)\} = \int_0^1t^b(1-t)dt$ and the binomial theorem, we have
\begin{align*}
    \E\left\{\frac{1}{(B+1)(B+2)}\right\}
    &=
    \int_0^1 (1-\lambda_c+\lambda_c t)^{n-2}(1-t)dt \\
    &=
    \frac{1}{\lambda_c^2}
    \left\{
    \frac{1-(1-\lambda_c)^{n-1}}{n-1}
    -
    \frac{1-(1-\lambda_c)^n}{n}
    \right\} \le \frac{1}{n(n-1)\lambda_{\min}^2}.
\end{align*}
Combining the above yields
\[
    \E\left[
        \{m_c(V)-m_c(V^{(r)})\}^2
        \mid X_i,X_j,C_i=c,C_j=d
    \right]
    \le
    \frac{2(L+1)^2}{n(n-1)\lambda_{\min}^2}.
\]
Therefore by Efron--Stein inequality,
\begin{align*}
    \var(m_c(C^n)\mid X_i,X_j,C_i=c,C_j=d)
    &\le
    \frac12\sum_{r=1}^{n-2}
    \E\left[
        \{m_c(V)-m_c(V^{(r)})\}^2
        \mid X_i,X_j,C_i=c,C_j=d
    \right] \\
    &\le \frac{n-2}{2} \cdot \frac{2(L+1)^2}{n(n-1)\lambda_{\min}^2} \leq \frac{(L+1)^2}{n\lambda_{\min}^2}.
\end{align*}
An identical bound holds for
$\var(m_d(C^n)\mid X_i,X_j,C_i=c,C_j=d)$. Hence by Cauchy--Schwarz,
\[
    \left|
    \cov(m_c(C^n),m_d(C^n)\mid X_i,X_j,C_i=c,C_j=d)
    \right|
    \le
    \frac{(L+1)^2}{n\lambda_{\min}^2}.
\]
uniformly over $c$ and $d$, and subsequently,
\[
    \left|
    \E\left[
        \cov\{m_{C_i}(C^n),m_{C_j}(C^n)\mid X_i,X_j,C_i,C_j\}
        \mid X_i,X_j
    \right]
    \right|
    \le
    \frac{(L+1)^2}{n\lambda_{\min}^2}.
\]
This bounds \eqref{eq:mp_cov_es1}.

To conclude bounding \eqref{eq:mp_lemma_cov}, we now bound \eqref{eq:mp_cov_es2}. Write
\[
    a_{c,d}
    \equiv
    \E\{m_c(C^n)\mid X_i,X_j,C_i=c,C_j=d\},
    \qquad
    b_{c,d}
    \equiv
    \E\{m_d(C^n)\mid X_i,X_j,C_i=c,C_j=d\}.
\]
Then the term of interest can be written as $\cov(a_{C_i,C_j},b_{C_i,C_j}\mid X_i,X_j)$.

Fix $c$ and $d,d'\in\mathcal C$. For a given realization $V=v$, let
\[
    T_{c,n}(v,d)\equiv T_{c,n}(C^n)
    \Bigm| X_i,X_j,C_i=c,C_j=d,V=v,
\]
\[
    N_{c,n}(v,d)\equiv N_{c,n}(C^n)
    \Bigm| X_i,X_j,C_i=c,C_j=d,V=v,
\]
and $m_c(v,d)\equiv T_{c,n}(v,d)/N_{c,n}(v,d)$ following our notation from before. 

Recall that, if we change only the value of $C_j$ while holding $V=v$ and $C_i=c$ fixed, then
\[
    |T_{c,n}(v,d)-T_{c,n}(v,d')|\le L,
    \qquad
    |N_{c,n}(v,d)-N_{c,n}(v,d')|\le 1.
\]
Hence
\[
    |m_c(V,d)-m_c(V,d')|
    \le
    \frac{L+1}{\min\{N_{c,n}(V,d),N_{c,n}(V,d')\}}.
\]
Also,
\[
    \min\{N_{c,n}(V,d),N_{c,n}(V,d')\}
    \ge
    1+\sum_{m\neq i,j}\one(C_m=c)
    =
    1+B,
\]
where $B\sim\textnormal{Bin}(n-2,\lambda_c)$. It follows from the calculations
above that
\[
    \sup_{d,d'}|a_{c,d}-a_{c,d'}|
    \le
    (L+1)\E\left(\frac{1}{B+1}\right)
    \le
    \frac{L+1}{(n-1)\lambda_{\min}}.
\]
Similarly,
\[
    \sup_{c,c'}|b_{c,d}-b_{c',d}|
    \le
    \frac{L+1}{(n-1)\lambda_{\min}}.
\]

Now let
\[
    \bar a_{C_i}(X_j)
    \equiv
    \sum_{d\in\mathcal C}P(C_j=d\mid X_j)a_{C_i,d},
    \qquad
    \bar b_{C_j}(X_i)
    \equiv
    \sum_{c\in\mathcal C}P(C_i=c\mid X_i)b_{c,C_j}.
\]
and write the term of our interest as
\begin{align*}
    &\cov(a_{C_i,C_j},b_{C_i,C_j}\mid X_i,X_j) \\
    &=
    \cov\left(
        a_{C_i,C_j}-\bar a_{C_i}(X_j)+\bar a_{C_i}(X_j),
        b_{C_i,C_j}-\bar b_{C_j}(X_i)+\bar b_{C_j}(X_i)
        \mid X_i,X_j
    \right) \\
    &=
    \cov\left(
        a_{C_i,C_j}-\bar a_{C_i}(X_j),
        b_{C_i,C_j}-\bar b_{C_j}(X_i)
        \mid X_i,X_j
    \right) \\
    &+
    \cov\left(
        a_{C_i,C_j}-\bar a_{C_i}(X_j),
        \bar b_{C_j}(X_i)
        \mid X_i,X_j
    \right) +
    \cov\left(
        \bar a_{C_i}(X_j),
        b_{C_i,C_j}-\bar b_{C_j}(X_i)
        \mid X_i,X_j
    \right),
\end{align*}
where $\cov\{\bar a_{C_i}(X_j),\bar b_{C_j}(X_i)\mid X_i,X_j\}=0$ since $C_i$ and $C_j$ are independent given $X_i,X_j$. Also,
\[
    \E\{a_{C_i,C_j}-\bar a_{C_i}(X_j)\mid X_i,X_j\}=0,
    \qquad
    \E\{b_{C_i,C_j}-\bar b_{C_j}(X_i)\mid X_i,X_j\}=0,
\]
and $0\le \bar a_{C_i}(X_j),\bar b_{C_j}(X_i)\le 1$. Therefore,
\begin{align*}
    \left|\cov(a_{C_i,C_j},b_{C_i,C_j}\mid X_i,X_j)\right|
    &\le
    \E\{|a_{C_i,C_j}-\bar a_{C_i}(X_j)|\mid X_i,X_j\} \\
    &+
    \E\{|b_{C_i,C_j}-\bar b_{C_j}(X_i)|\mid X_i,X_j\} \\
    &+
    \E\{|a_{C_i,C_j}-\bar a_{C_i}(X_j)|
          |b_{C_i,C_j}-\bar b_{C_j}(X_i)|
          \mid X_i,X_j\}.
\end{align*}
Now
\[
    |a_{C_i,C_j}-\bar a_{C_i}(X_j)|
    \le
    \sup_{d,d'}|a_{C_i,d}-a_{C_i,d'}|,
\]
which implies
\[
    \E\{|a_{C_i,C_j}-\bar a_{C_i}(X_j)|\mid X_i,X_j\}
    \le
    \frac{L+1}{(n-1)\lambda_{\min}}.
\]
Similarly,
\[
    \E\{|b_{C_i,C_j}-\bar b_{C_j}(X_i)|\mid X_i,X_j\}
    \le
    \frac{L+1}{(n-1)\lambda_{\min}},
\]
and
\[
    \E\{|a_{C_i,C_j}-\bar a_{C_i}(X_j)|
          |b_{C_i,C_j}-\bar b_{C_j}(X_i)|
          \mid X_i,X_j\}
    \le
    \frac{(L+1)^2}{(n-1)^2\lambda_{\min}^2}.
\]
Combining the above bounds yields
\[
    \left|\cov(a_{C_i,C_j},b_{C_i,C_j}\mid X_i,X_j)\right|
    \le
    \frac{2(L+1)}{(n-1)\lambda_{\min}}
    +
    \frac{(L+1)^2}{(n-1)^2\lambda_{\min}^2}.
\]
This bounds \eqref{eq:mp_cov_es2}. Hence the combined bound on \eqref{eq:mp_lemma_cov} is $O(n^{-1})$.

As a final piece, we bound \eqref{eq:mp_lemma_mean_mismatch} and show it is $O(n^{-2})$. Using the same $a_{c,d}$ as above, iterated expectation gives
\[
    \E(Z_i\mid X_i,X_j)
    =
    \sum_{c\in\mathcal C}\sum_{d\in\mathcal C}
    P(C_i=c\mid X_i)P(C_j=d\mid X_j)a_{c,d},
\]
and
\[
    \pi_n(X_i)
    =
    \sum_{c\in\mathcal C}\sum_{d\in\mathcal C}
    P(C_i=c\mid X_i)\lambda_d a_{c,d}.
\]
It follows that
\[
    \E(Z_i\mid X_i,X_j)-\pi_n(X_i)
    =
    \sum_{c\in\mathcal C}P(C_i=c\mid X_i)
    \sum_{d\in\mathcal C}
    \{P(C_j=d\mid X_j)-\lambda_d\}
    (a_{c,d}-\tilde a_c),
\]
where $\tilde a_c\equiv \sum_{d\in\mathcal C}\lambda_d a_{c,d}$. Hence,
\begin{align*}
    \left|\E(Z_i\mid X_i,X_j)-\pi_n(X_i)\right|
    &\le
    \sum_{c\in\mathcal C}P(C_i=c\mid X_i)
    \left(
        \sum_{d\in\mathcal C}
        |P(C_j=d\mid X_j)-\lambda_d|
    \right)
    \sup_{d,d'}|a_{c,d}-a_{c,d'}| \\
    &\le
    2\sum_{c\in\mathcal C}P(C_i=c\mid X_i)
    \sup_{d,d'}|a_{c,d}-a_{c,d'}|  \leq \frac{2(L+1)}{(n-1)\lambda_{\min}},
\end{align*}
where the second inequality follows because $P(C_j=\cdot\mid X_j)$ and
$\lambda_\cdot$ are probability vectors. Therefore, we also have
\[
    \left|\E(Z_j\mid X_i,X_j)-\pi_n(X_j)\right|
    \le
    \frac{2(L+1)}{(n-1)\lambda_{\min}}.
\]
Consequently,
\[
    \left|
    \{\E(Z_i\mid X_i,X_j)-\pi_n(X_i)\}
    \{\E(Z_j\mid X_i,X_j)-\pi_n(X_j)\}
    \right|
    \le
    \frac{4(L+1)^2}{(n-1)^2\lambda_{\min}^2}.
\]
This shows that \eqref{eq:mp_lemma_mean_mismatch} is $O(n^{-2})$, which
completes the proof.
\end{proof}

\begin{proof}[Proof of Theorem \ref{thm:asymp_var_DR}]
Throughout the proof, we assume $n \geq n_0$ for $n_0 \in \mathbb N$ so that positivity in Theorem \ref{thm:identification_exo} holds. Define the \emph{oracle} influence function as
\begin{align*}
\phi_i &\equiv \mu_1(X_i)-\mu_0(X_i) +\frac{Z_i}{\pi_n(X_i)}\{Y_i-\mu_1(X_i)\} -\frac{1-Z_i}{1-\pi_n(X_i)}\{Y_i-\mu_0(X_i)\}. 
\end{align*}

We first decompose $\widehat\psi^{\mathrm{DR}}_{\ATE}$ with respect to $\phi_i$. Let
\[
g(x)\equiv\frac{\widehat\mu_1(x)-\mu_1(x)}{\pi_n(x;\theta)}+\frac{\widehat\mu_0(x)-\mu_0(x)}{1-\pi_n(x;\theta)}.
\]
Some algebra gives 
\begin{equation*}
\widehat\psi^{\mathrm{DR}}_{\ATE}
= \frac1n\sum_{i=1}^n \phi_i - \frac1n\sum_{i=1}^n \{Z_i-\pi_n(X_i)\}g(X_i) \equiv \frac1n\sum_{i=1}^n \phi_i - R_n,
\end{equation*}
so we have
\begin{equation}\label{eq:var_decomp}
    n\var(\widehat\psi^{\mathrm{DR}}_{\ATE}) = n\var\left(\frac1n\sum_{i=1}^n \phi_i\right) - 2n\cov\left(\frac1n\sum_{i=1}^n \phi_i,R_n\right) + n\var(R_n).
\end{equation}

We first show that $n\var(R_n)=o_{\bP}(1)$. Note that
\[
\E(R_n)
=\frac1n\sum_{i=1}^n \E\big[g(X_i)\{\E(Z_i\mid X_i)-\pi_n(X_i)\}\big]=0,
\]
so $\var(R_n)=\E(R_n^2)$. Now write
\begin{equation*}
n\E(R_n^2)
=\frac1n\sum_{i=1}^n \E\big[g(X_i)^2\{Z_i-\pi_n(X_i)\}^2\big]
+\frac{2}{n}\sum_{i<j}\E\big[g(X_i) g(X_j)\{Z_i-\pi_n(X_i)\}\{Z_j-\pi_n(X_j)\}\big].
\end{equation*}

For the diagonal term, $\{Z_i-\pi_n(X_i)\}^2 \leq 1$ implies
\begin{align*}
\frac1n\sum_{i=1}^n \E\big[g(X_i)^2\{Z_i-\pi_n(X_i)\}^2\big]
\leq \frac1n\sum_{i=1}^n \E\{g(X_i)^2\} =\E\{g(X)^2\}.
\end{align*}

For the off-diagonal term, using Lemma \ref{lem:remainder_var_multi} and $2|ab|\leq a^2+b^2$, can be bounded as
\begin{align*}
&\frac{2}{n}\sum_{i<j}\Big|\E\big[g(X_i)g(X_j)\{Z_i-\pi_n(X_i)\}\{Z_j-\pi_n(X_j)\}\big]\Big| \\
&\leq \frac{2}{n}\sum_{i<j}\E\big(|g(X_i)g(X_j)|\big|\E\big[\{Z_i-\pi_n(X_i)\}\{Z_j-\pi_n(X_j)\} \mid X_i, X_j\big]\big|\big) \\
&\leq \frac{2}{n}\sum_{i<j} \E\Big(|g(X_i)g(X_j)|\cdot \frac{C}{n}\ \Big)
\leq \frac{C}{n^2}\E\Big[\sum_{i<j}\{g(X_i)^2+g(X_j)^2\} \Big]
\leq C\E\{g(X)^2\},
\end{align*}
and therefore
\[
n\E(R_n^2)\le (1+C)\E\{g(X)^2\}.
\]
Finally, positivity implies
\[
g(X)^2 \leq \frac{2}{\gamma^2}\Big[\{\widehat\mu_1(X)-\mu_1(X)\}^2+\{\widehat\mu_0(X)-\mu_0(X)\}^2\Big],
\]
so $\E\{g(X)^2\}=o_{\bP}(1)$ by $\|\widehat\mu_z-\mu_z\|_{L^2(\bP)}=o_{\bP}(1)$ for $z\in\{0,1\}$.
Thus we have $n\E(R_n^2)=o_{\bP}(1)$.

Next we analyze the leading term in \eqref{eq:var_decomp}. Write
\[
n\var\left(\frac1n\sum_{i=1}^n \phi_i\right)
=
\frac{1}{n}\sum_{i=1}^n\var(\phi_i)
+
\frac{2}{n}\sum_{i<j}\cov(\phi_i,\phi_j).
\]
The analysis of the diagonal term is standard; by the law of total variance,
\[
\var(\phi_i)
=
\E\{\var(\phi_i\mid X_i)\}
+\var\{\E(\phi_i\mid X_i)\}.
\]
Note that $\var\{\E(\phi_i\mid X_i)\} = \var(\mu_1(X_i) - \mu_0(X_i))$ since $\E(\phi_i\mid X_i)= \mu_1(X_i) - \mu_0(X_i)$, and
\[
\var(\phi_i\mid X_i)
=
\frac{\var(Y_i \mid Z_i=1, X_i)}{\pi_n(X_i)}
+
\frac{\var(Y_i \mid Z_i=0, X_i)}{1-\pi_n(X_i)}.
\]
Therefore
\begin{align*}
\var(\phi_i) = \E\left\{\frac{\var(Y_i \mid Z_i=1, X_i)}{\pi_n(X_i)} + \frac{\var(Y_i \mid Z_i=0, X_i)}{1-\pi_n(X_i)}\right\} + \var(\mu_1(X_i) - \mu_0(X_i)),
\end{align*}
and applying Theorem \ref{thm:propensity} and dominated convergence under positivity yield
\begin{align*}
\E\left\{
\frac{\var(Y_i \mid Z_i=1, X_i)}{\pi_n(X_i)} + \frac{\var(Y_i \mid Z_i=0, X_i)}{1-\pi_n(X_i)}
\right\} \to \E\left\{ \frac{\var(Y_i \mid Z_i=1, X_i)}{\pi(X_i;\theta)} + \frac{\var(Y_i \mid Z_i=0, X_i)}{1-\pi(X_i;\theta)} \right\}.
\end{align*}

We now show that the covariance term is zero. Define
\[
B_i
\equiv
\frac{Y_i^1-\mu_1(X_i)}{\pi_n(X_i)}
+
\frac{Y_i^0-\mu_0(X_i)}{1-\pi_n(X_i)}.
\]
Some algebra gives
\begin{equation*}
\phi_i = (Y_i^1-Y_i^0)+B_i(Z_i-\pi_n(X_i)).
\end{equation*}
Now fix $i\neq j$. Expanding the covariance, we have
\begin{align}
\cov(\phi_i,\phi_j)
&=
\cov(Y_i^1-Y_i^0,Y_j^1-Y_j^0). \label{eq:cov_first} \\ 
&+ \cov(Y_i^1-Y_i^0,B_j\{Z_j-\pi_n(X_j)\}) \label{eq:cov_second} \\
&+ \cov(B_i\{Z_i-\pi_n(X_i)\},Y_j^1-Y_j^0) \label{eq:cov_third} \\
&+ \cov(B_i\{Z_i-\pi_n(X_i)\},B_j\{Z_j-\pi_n(X_j)\}) \label{eq:cov_fourth}.
\end{align}
\eqref{eq:cov_first} is zero by the iid assumption. We show that the remaining terms are also zero by conditioning on $H_n \equiv (X^n,A^n,Q^n)$.

Consider \eqref{eq:cov_second}. We have
\begin{align*}
\cov(Y_i^1-Y_i^0,B_j\{Z_j-\pi_n(X_j)\})
&=
\E\left[\cov\big(Y_i^1-Y_i^0,B_j\{Z_j-\pi_n(X_j)\}\mid H_n\big)\right] \\
&+ \cov\left(\E(Y_i^1-Y_i^0\mid H_n),\E\big(B_j\{Z_j-\pi_n(X_j)\}\mid H_n\big)\right).
\end{align*}
Note that $\{Z_j-\pi_n(X_j)\}$ is constant given $H_n$ since $Z^n$ is then a deterministic function of the design; further, $(Y_i^1-Y_i^0)\independent B_j$ since $B_j$ is a function of $Y_j^1, Y_j^0$, and $X_j$. It follows that
\[
\cov\big(Y_i^1-Y_i^0,B_j\{Z_j-\pi_n(X_j)\}\mid H_n\big)
=\{Z_j-\pi_n(X_j)\}\cov\big(Y_i^1-Y_i^0,B_j\mid H_n\big)=0.
\]
Moreover, the covariance of the conditional means is also zero as
\[
\E\big(B_j\{Z_j-\pi_n(X_j)\}\mid H_n\big)
=\{Z_j-\pi_n(X_j)\}\E(B_j\mid H_n)= \{Z_j-\pi_n(X_j)\}\E(B_j\mid X_j) = 0.
\]
Hence by the same logic, \eqref{eq:cov_third} is also zero.

Next, consider \eqref{eq:cov_fourth}. Again, write
\begin{align*}
\cov(B_i&\{Z_i-\pi_n(X_i)\}, B_j\{Z_j-\pi_n(X_j)\}) \\ 
&=
\E\left[\cov\big(B_i\{Z_i-\pi_n(X_i)\},B_j\{Z_j-\pi_n(X_j)\}\mid H_n\big)\right] \\
&+ \cov\left(\E\big(B_i\{Z_i-\pi_n(X_i)\}\mid H_n\big),\E\big(B_j\{Z_j-\pi_n(X_j)\}\mid H_n\big)\right).
\end{align*}
$\{Z_i-\pi_n(X_i)\}$ and $\{Z_j-\pi_n(X_j)\}$ are fixed given $H_n$, so the conditional covariance reduces to
\[
\{Z_i-\pi_n(X_i)\}\{Z_j-\pi_n(X_j)\}\cov(B_i,B_j\mid H_n)=0,
\]
since $B_i\independent B_j$. Also,
\[
\E\big(B_i\{Z_i-\pi_n(X_i)\}\mid H_n\big)
=\{Z_i-\pi_n(X_i)\}\E(B_i\mid H_n)=0,
\]
and similarly for $j$, so the covariance of the conditional means is zero. 

Combining the above, we conclude $\cov(\phi_i,\phi_j)=0$ for all $i\neq j$. Therefore,
\[
n\var\left(\frac1n\sum_{i=1}^n \phi_i\right)
=
\frac{1}{n}\sum_{i=1}^n\var(\phi_i),
\]

The proof concludes by noting that the second term in \eqref{eq:var_decomp} is
\[
\big|2n\cov(n^{-1}\textstyle\sum_i \phi_i,R_n)\big| \leq 2\sqrt{n\var(n^{-1}\textstyle\sum_i \phi_i)}\sqrt{n\var(R_n)} = o_\bP(1),
\]
by Cauchy-Schwarz, $n\var(n^{-1}\sum_i \phi_i)=O_\bP(1)$, and $n\var(R_n)=o_\bP(1)$.
\end{proof}

\subsection{Proof of Theorem \ref{thm:asymp_var_plm}}

We first shed light on how $\widehat\psi_Z^\star$ is motivated. Note that residualizing the PLIV model \eqref{eq:plm_itt} with respect to $X$ gives the conditional moment restriction
\[
    \E\Big[Y-m(X)-\psi_Z\{Z-\pi_n(X;\theta,b_n)\} \Bigm| X,Q\Big] = 0.
\]
For any integrable function $f(X,Q)$, multiplying the above equation by $f(X,Q)$ yields the unconditional moment
\[
    \E\Big(f(X,Q)\big[Y-m(X)-\psi_Z\{Z-\pi_n(X;\theta,b_n)\}\big]\Big)=0,
\]
which implies
\[
\psi_Z = \frac{\E\big[f(X,Q)\{Y-m(X)\}\big]}
         {\E\big[f(X,Q)\{Z-\pi_n(X;\theta,b_n)\}\big]},
\]
provided the denominator is nonzero. Although any $f(X,Q)$ yields such an unconditional moment, only its residualized component with respect to $X$ contributes after residualizing $Y$ and $Z$. We therefore restrict without loss of generality to functions satisfying $\E\{f(X,Q)\mid X\}=0$. For such functions,
\[
    \E\big[f(X,Q)\{Z-\pi_n(X;\theta,b_n)\}\big] = \E\big\{f(X,Q)Z\big\}.
\]
Hence any $f(X,Q)$ satisfying $\E\{f(X,Q)\mid X\}=0$ and $\E\{f(X,Q)Z\}\neq 0$ is a valid instrument under our queue-randomizing design. The corresponding sample analogue gives an estimator
\[
    \widehat\psi_Z(f) =
    \frac{\Pn\big[f(X,Q)\{Y-\widehat m(X)\}\big]}
         {\Pn\big[f(X,Q)\{Z-\pi_n(X;\theta,b_n)\}\big]}.
\]
This formulation allows us to use the efficiency theory for semiparametric conditional moment restrictions \citep{chamberlain1992efficiency, ai2003efficient}; although many instruments identify $\psi_Z$, only the optimal instrument attains the efficiency bound. In the theorem below, we first derive the asymptotic variance of $\widehat\psi_Z(f)$ for a generic uniformly bounded $f$ satisfying $\E\{f(X,Q)\mid X\}=0$ and $\E\{f(X,Q)Z\}\neq 0$, and then show that the oracle choice $f(X,Q)=r_n(X,Q)/\sigma(X)$ yields the iid efficiency bound, with the feasible estimator $\widehat\psi_Z^\star$ obtained by replacing $\sigma(X)$ with $\widehat\sigma(X)$ and attaining the same asymptotic variance.

\begin{proof}[Proof of Theorem \ref{thm:asymp_var_plm}]
Throughout the proof, we write
\begin{equation*}
    f_i \equiv f(X_i,Q_i), \quad
    r_i \equiv r_n(X_i, Q_i), \quad 
    \zeta_i \equiv \alpha_{Q_i}(\beta,p)-\pi(X_i;\theta), \quad 
    \tilde m(x) \equiv \widehat m(x)-m(x),
\end{equation*}
for notational ease. 

First, note that $m(X_i) = \psi_Z\pi_n(X_i)+g(X_i)$ since $\E(U_i\mid X_i)=0$ by iterated expectation. Then we can write
\begin{align*}
    Y_i-\widehat m(X_i) &= Y_i - m(X_i) + m(X_i) - \widehat m(X_i) \\
    &= \psi_Z Z_i+g(X_i)+U_i - m(X_i) + m(X_i) -\widehat m(X_i) \\
    &= \psi_Z \{Z_i-\pi_n(X_i)\}+U_i-\tilde m(X_i),
\end{align*}
which implies
\begin{align*}
    \frac{1}{n}\sum_{i=1}^n f_i\{Y_i-\widehat m(X_i)\} &= \psi_Z\frac{1}{n}\sum_{i=1}^n f_i\{Z_i-\pi_n(X_i)\} + \frac{1}{n}\sum_{i=1}^n f_iU_i - \frac{1}{n}\sum_{i=1}^n f_i\tilde m(X_i) \\
    &\equiv \psi_Z\Delta_n + S_n - R_n.
\end{align*}
Subsequently, 
\[
    \widehat\psi_Z-\psi_Z = \frac{S_n - R_n}{\Delta_n}.
\]

We first show $\Delta_n = \E(f_i\zeta_i) + o_\bP(1) \equiv \Delta + o_\bP(1)$. Decompose $\Delta_n$ as
\begin{align*}
    \Delta_n &= \frac{1}{n}\sum_{i=1}^n f_i\{Z_i - \tilde\pi_n(X_i, Q_i) + \tilde\pi_n(X_i, Q_i) - \pi_n(X_i)\} \\
    &= \frac{1}{n}\sum_{i=1}^n f_i\{Z_i- \tilde\pi_n(X_i, Q_i) + \frac{1}{n}\sum_{i=1}^n f_ir_i \equiv B_n + C_n.
\end{align*}
Consider $B_n$. Note that $\E(B_n)=0$ since $\E\{Z_i-\tilde\pi_n(X_i, Q_i)\mid X_i,Q_i\}=0$, so $n\var(B_n) = n\E(B_n^2)$. Now write
\begin{align*}
    n\E(B_n^2)
    &= \frac{1}{n}\sum_{i=1}^n\E\Big[f_i^2\{Z_i-\tilde\pi_n(X_i, Q_i)\}^2\Big] + \frac{2}{n}\sum_{i<j}\E\Big[f_if_j\{Z_i-\tilde\pi_n(X_i, Q_i)\}\{Z_j-\tilde\pi_n(X_j, Q_j)\}\Big].
\end{align*}
The diagonal term is bounded by $\E(f_i^2)$ since $\{Z_i-\tilde\pi_n(X_i, Q_i)\}^2\le 1$. For the off-diagonal term, we introduce the following lemma, whose proof is similar to that of Lemma \ref{lem:remainder_var_multi} and thus omitted.

\begin{lemma}\label{lem:centered_z_cov}
There exists a constant $C>0$ depending only on $p$ such that for all $n \geq 2$ and all $i \neq j$,
\begin{equation*}
    \left|
    \E\Big[
        \left\{Z_i-\tilde\pi_n(X_i,Q_i)\right\}
        \left\{Z_j-\tilde\pi_n(X_j,Q_j)\right\}
        \middle| X_i,Q_i,X_j,Q_j
    \Big]
    \right|
    \leq \frac{C}{n} \quad \text{a.s.}
\end{equation*}
\end{lemma}


Lemma \ref{lem:centered_z_cov} and the uniform boundedness of $f$ imply
\begin{align*}
    &\frac{2}{n}\sum_{i<j}
    \left|
    \E\Big[f_if_j\{Z_i-\tilde\pi_n(X_i,Q_i)\}\{Z_j-\tilde\pi_n(X_j,Q_j)\}\Big]
    \right|\\
    &\le
    \frac{2}{n}\sum_{i<j}
    \E\Big(
        |f_if_j|
        \left|
        \E\Big[
            \{Z_i-\tilde\pi_n(X_i,Q_i)\}\{Z_j-\tilde\pi_n(X_j,Q_j)\}
            \middle| X_i,Q_i,X_j,Q_j
        \Big]
        \right|
    \Big)\\
    &\le
    \frac{2c}{n^2}\sum_{i<j}\E|f_if_j|=O(1).
\end{align*}
Hence $n\var(B_n) = O(1)$ and, in turn, $B_n= O_{\bP}(n^{-1/2}) = o_\bP(1)$ by Chebyshev's inequality. 

For $C_n$, write
\[
    C_n = \{C_n-\E(f_ir_i)\} + \E\{f_i(r_i-\zeta_i)\} + \Delta.
\]
Since $(X_i,Q_i)$ are iid across $i$, $f_ir_i$ are iid. Further, the fact that $|r_i| \leq 1$ and $f_i$ is uniformly bounded implies $\var(f_ir_i)$ is also uniformly bounded. It thus follows that the first term is $o_\bP(1)$. For the second term, $\E\{f_i(r_i-\zeta_i)\} \to 0$ a.s. by Theorem \ref{thm:propensity}. It follows that
\begin{equation*}
    C_n = \Delta + o_\bP(1),
\end{equation*}
and consequently $\Delta_n = B_n + C_n = \Delta + o_\bP(1)$.

We next show
\begin{equation*}\label{eq:plm_thm_eq2}
    n\var(S_n - R_n) = n\var(S_n) + n\var(R_n) - 2n\cov(S_n, R_n) = \E(f_i^2U_i^2) + o_\bP(1).
\end{equation*}
First consider $n\var(S_n)$. Since $\E(f_iU_i)=0$, we have
\begin{align*}
    n\var(S_n) &=
    \frac{1}{n}\sum_{i=1}^n \var(f_iU_i)
    + \frac{2}{n}\sum_{i<j}\cov(f_iU_i,f_jU_j) \\
    &= \frac{1}{n}\sum_{i=1}^n \E(f_i^2U_i^2) + \frac{2}{n}\sum_{i<j}\cov(f_iU_i,f_jU_j) \\
    &= \E(f_i^2U_i^2),
\end{align*}
where the third equality follows from 
\[
    \cov(f_iU_i,f_jU_j) = \E(f_iU_if_jU_j) = \E\{f_if_j\E(U_i \mid X_i,Q_i)\E(U_j \mid X_j,Q_j)\} = 0.
\] 
Since $\tilde m$ is fixed conditional on $\mathcal{D}$, we can apply the same logic to show $n\var(R_n) = \var(f_i\tilde m(X_i))$. Then the uniform boundedness of $f_i$ and $\|\widehat m - m \|_{L_2(\bP)} = o_\bP(1)$ imply
\begin{equation*}
    n\var(R_n) = \var(f_i\tilde m(X_i)) \leq \E\{f_i^2\tilde m(X_i)^2\} \leq c'^2\|\tilde m\|_{L_2(\bP)}^2 = o_\bP(1).
\end{equation*}
Further,
\begin{equation*}
    \cov(S_n,R_n) = \frac{1}{n^2}\sum_{i=1}^n \cov(f_iU_i,f_i\tilde m(X_i)) + \frac{2}{n^2}\sum_{i<j}\cov(f_iU_i,f_j\tilde m(X_j)) = 0,
\end{equation*}
since $\cov(f_iU_i,f_i\tilde m(X_i)) =\E\{f_i^2\tilde m(X_i)\E(U_i \mid X_i, Q_i)\} = 0$ and $\cov(f_iU_i,f_j\tilde m(X_j)) = 0$ by independence across units. 

Putting what we have together, since $\Delta_n = \Delta + o_\bP(1)$ and $n\var(S_n - R_n) = O_{\bP}(1)$, continuous mapping theorem and Chebyshev's inequality imply $\Delta_n^{-1} = \Delta^{-1} + o_\bP(1)$ and $S_n-R_n = O_\bP(n^{-1/2})$, respectively. It follows that
\[
    \frac{S_n - R_n}{\Delta_n} = \frac{S_n - R_n}{\Delta} + (S_n - R_n)\bigg(\frac{1}{\Delta_n} - \frac{1}{\Delta}\bigg) = \frac{S_n - R_n}{\Delta} + o_\bP(n^{-1/2}),
\]
and therefore
\[
    n\var(\widehat\psi_Z) = n\var(\widehat\psi_Z - \psi) = \frac{n\var(S_n-R_n)}{\Delta^2} + o_{\bP}(1) = \frac{\E(f_i^2U_i^2)}{\E(f_i\zeta_i)^2} +  o_{\bP}(1) = \frac{\E\{f_i^2\sigma(X_i)\}}{\E(f_i\zeta_i)^2} +  o_{\bP}(1),
\]
where the last equality follows because, under our queue-randomizing design, $Q_i$ is assigned using only $X_i$ and independent design randomness. Hence $Q_i \independent U_i \mid X_i$, so
\[
    \E(U_i^2\mid X_i,Q_i)=\E(U_i^2\mid X_i)=\sigma(X_i).
\]

Now consider the oracle choice $f_i=r_i/\sigma(X_i)$. Applying Theorem~\ref{thm:propensity} gives
\[
    \lim_{n\to\infty} n\var\{\widehat\psi_Z)\}
    =
    \left(\E\left[
        \frac{\{\alpha_Q(\beta,p)-\pi(X;\theta)\}^2}{\sigma(X)}
    \right]\right)^{-1}.
\]
Finally, replacing $\sigma(X)$ by $\widehat\sigma(X)$ does not change the limit under the assumed consistency of $\widehat\sigma$. Hence 
\[
    \lim_{n \to \infty} n\var(\widehat\psi_Z^\star)
    =
    \left(\E\left[
        \frac{\{\alpha_Q(\beta,p)-\pi(X;\theta)\}^2}{\sigma(X)}
    \right]\right)^{-1}. \qedhere
\]
\end{proof}

\subsection{Proof of Theorem 7}
\begin{proof}
Refer to Section \ref{sec:opt_detail} for details.    
\end{proof}

\subsection{Proof of Theorem \ref{thm:asymp_normal_DR}}
\begin{proof}
Throughout the proof, we assume $n \geq n_0$ for $n_0 \in \mathbb N$ to ensure positivity holds. As in the proof of Theorem \ref{thm:asymp_var_DR}, denote the \emph{oracle} influence function as
\begin{align*}
  \phi_i &\equiv \mu_1(X_i)-\mu_0(X_i) + \frac{Z_i}{\pi_n(X_i)}\{Y_i-\mu_1(X_i)\} - \frac{1-Z_i}{1-\pi_n(X_i)}\{Y_i-\mu_0(X_i)\} \\
  &= \mu_1(X_i)-\mu_0(X_i) + \frac{Z_i}{\pi_n(X_i)}\{Y_i^1-\mu_1(X_i)\} - \frac{1-Z_i}{1-\pi_n(X_i)}\{Y_i^0-\mu_0(X_i)\},
\end{align*}
where the equality holds from consistency.

Since
\[
\sqrt{n}\big(\widehat\psi_{\ATE}^{\DR}-\psi_{\ATE}\big)
=
\frac{1}{\sqrt{n}}\sum_{i=1}^n\big(\phi_i-\psi_{\ATE}\big)
+ o_\bP(1),
\]
it suffices to show that $n^{-1/2}\sum_{i=1}^n(\phi_i-\psi_{\ATE})$ converges in distribution to $\cN(0,V(\theta))$. 

Write
\begin{align*}
\phi_i-\psi_{\ATE}
&= \big\{\mu_1(X_i)-\mu_0(X_i)-\psi_{\ATE}\big\}
+ \frac{Z_i}{\pi_n(X_i)}\{Y_i^1-\mu_1(X_i)\}
-\frac{1-Z_i}{1-\pi_n(X_i)}\{Y_i^0-\mu_0(X_i)\} \\
&\equiv \big\{\mu_1(X_i)-\mu_0(X_i)-\psi_{\ATE}\big\} + R_{i,n}.
\end{align*}
Then
\[
\frac{1}{\sqrt{n}}\sum_{i=1}^n(\phi_i-\psi_{\ATE})
= \frac{1}{\sqrt{n}}\sum_{i=1}^n R_{i,n} + \frac{1}{\sqrt{n}}\sum_{i=1}^n\big\{\mu_1(X_i)-\mu_0(X_i)-\psi_{\ATE}\big\} 
\equiv T_{1,n} + T_{2,n}.
\]
Write
\[
V(\theta) = \E\bigg\{\frac{\var(Y_i^1\mid X_i)}{\pi(X_i;\theta)}
+\frac{\var(Y_i^0\mid X_i)}{1-\pi(X_i;\theta)}\bigg\}
+\var\big(\mu_1(X_i)-\mu_0(X_i)\big) \equiv V_1(\theta) + V_2.
\]

Since $X_i$'s are iid, together with $\E\{\mu_1(X_i)-\mu_0(X_i)\}=\psi_{\ATE}$ and $\E\big[\{\mu_1(X_i)-\mu_0(X_i)\}^2\big]<\infty$, we have by the canonical central limit theorem (CLT) that
\[
T_{2,n} \rightsquigarrow \cN\left(0,V_2\right).
\]

For $T_{1,n}$, let $H_n \equiv (X^n,A^n,Q^n)$ as in the proof of Theorem \ref{thm:asymp_var_DR}. We verify the conditional Lyapunov condition
\begin{equation*}
\frac{1}{s_n^{2+\delta}}\sum_{i=1}^n \E(|R_{i,n}|^{2+\delta} \mid H_n) \overset{p}{\to} 0,
\end{equation*}
where $\delta=2$ and $s_n^2 \equiv\sum_{i=1}^n \var(R_{i,n}\mid H_n)$, since $R_{1,n},\dots,R_{n,n}$ are independent conditional on $H_n$, as $Z^n$ is then a deterministic function of $(A^n,Q^n)$ under our design and the potential outcomes are independent across units. 

Note that 
\[
s_n^2
=
\sum_{i=1}^n\left\{
\frac{Z_i\var(Y_i^1\mid X_i)}{\pi_n(X_i)^2}
+\frac{(1-Z_i)\var(Y_i^0\mid X_i)}{\{1-\pi_n(X_i)\}^2}
\right\},
\]
since 
\begin{align*}
\var(R_{i,n}\mid H_n)
&= \E(R_{i,n}^2\mid H_n) \\
&= \E\Bigg(
\left[ \frac{Z_i}{\pi_n(X_i)}\{Y_i^1-\mu_1(X_i)\} - \frac{1-Z_i}{1-\pi_n(X_i)}\{Y_i^0-\mu_0(X_i)\}
\right]^2 \biggm| H_n \Bigg) \\
&=
\frac{Z_i^2}{\pi_n(X_i)^2}
\E\left[\{Y_i^1-\mu_1(X_i)\}^2\mid H_n\right]
+
\frac{(1-Z_i)^2}{\{1-\pi_n(X_i)\}^2}
\E\left[\{Y_i^0-\mu_0(X_i)\}^2\mid H_n\right] \\
&\quad
-
\frac{2Z_i(1-Z_i)}{\pi_n(X_i)\{1-\pi_n(X_i)\}}
\E\left[\{Y_i^1-\mu_1(X_i)\}\{Y_i^0-\mu_0(X_i)\}\mid H_n\right] \\
&=
\frac{Z_i\var(Y_i^1\mid X_i)}{\pi_n(X_i)^2}
+
\frac{(1-Z_i)\var(Y_i^0\mid X_i)}{\{1-\pi_n(X_i)\}^2}.
\end{align*}

Now define
\[
M_n \equiv \sum_{i=1}^n \E\big(|R_{i,n}|^4 \mid H_n\big),
\]
and note that
\begin{align*}
|R_{i,n}|^{4}
&\leq
8 \left| \frac{Z_i}{\pi_n(X_i)}\{Y_i^1-\mu_1(X_i)\} \right|^{4} +
8 \left| \frac{1-Z_i}{1-\pi_n(X_i)}\{Y_i^0-\mu_0(X_i)\} \right|^{4} \\
&\leq
\frac{8}{\gamma^4}
\Big( |Y_i^1-\mu_1(X_i)|^{4} + |Y_i^0-\mu_0(X_i)|^{4} \Big),
\end{align*}
where the first inequality follows from $|a-b|^{2+\delta} \leq 2^{1+\delta}(|a|^{2+\delta}+|b|^{2+\delta})$ with $\delta=2$, and the second from positivity. It follows that
\[
\E\big(|R_{i,n}|^{4}\mid H_n\big)
\leq
\frac{8}{\gamma^4}\Big\{\E\big(|Y_i^1-\mu_1(X_i)|^{4}\mid X_i\big) + \E\big(|Y_i^0-\mu_0(X_i)|^{4}\mid X_i\big) \Big\},
\]
and therefore
\begin{align*}
\frac{M_n}{n} \le \frac{8/\gamma^4}{n} \sum_{i=1}^n \Big\{\E\big(|Y_i^1-\mu_1(X_i)|^{4}\mid X_i\big) + \E\big(|Y_i^0-\mu_0(X_i)|^{4}\mid X_i\big) \Big\} = O_\bP(1).
\end{align*}
by the weak law of large numbers and the moment assumption.

Now decompose
\begin{align*}
\frac{s_n^2}{n}
&=
\frac1n \sum_{i=1}^n
\left\{
\frac{Z_i \var(Y_i^1 \mid X_i)}{\pi_n(X_i)^2}
+
\frac{(1-Z_i)\var(Y_i^0 \mid X_i)}{\{1-\pi_n(X_i)\}^2}
\right\} \\
&=
\frac1n \sum_{i=1}^n
\left\{
\frac{\var(Y_i^1 \mid X_i)}{\pi_n(X_i)}
+
\frac{\var(Y_i^0 \mid X_i)}{1-\pi_n(X_i)}
\right\} \\
&\quad +
\frac1n \sum_{i=1}^n
\{Z_i-\pi_n(X_i)\}
\left\{
\frac{\var(Y_i^1 \mid X_i)}{\pi_n(X_i)^2}
-
\frac{\var(Y_i^0 \mid X_i)}{\{1-\pi_n(X_i)\}^2}
\right\} \\
&\equiv A_n + B_n.
\end{align*}
We first show $A_n \overset{p}{\to} V_1(\theta)$. Let
\[
a_n(X_i)
\equiv
\frac{\var(Y_i^1 \mid X_i)}{\pi_n(X_i)}
+
\frac{\var(Y_i^0 \mid X_i)}{1-\pi_n(X_i)},
\]
so that
\[
A_n = \frac1n \sum_{i=1}^n a_n(X_i).
\]
Then
\[
A_n - V_1(\theta)
=
\left[\frac1n \sum_{i=1}^n a_n(X_i) - \E\{a_n(X_i)\}\right] + \left[ \E\{a_n(X_i)\} - V_1(\theta) \right].
\]
Since $a_n(X_1),\dots,a_n(X_n)$ are iid, positivity and the moment condition imply
\[
\frac1n \sum_{i=1}^n a_n(X_i) \overset{p}{\to} \E\{a_n(X_i)\},
\]
by the weak law of large numbers. Then applying Theorem \ref{thm:propensity} and dominated convergence yields
\[
\E\{a_n(X_i)\} \to V_1(\theta),
\]
and subsequently $A_n \overset{p}{\to} V_1(\theta)$.

It remains to show that $B_n = o_{\bP}(1)$, whose proof closely follows that of $n\var(R_n)=o_{\bP}(1)$ in Theorem \ref{thm:asymp_var_DR}, thanks to Lemma \ref{lem:remainder_var_multi}. Let
\[
h_n(X_i)
\equiv
\frac{\var(Y_i^1 \mid X_i)}{\pi_n(X_i)^2}
-
\frac{\var(Y_i^0 \mid X_i)}{\{1-\pi_n(X_i)\}^2},
\]
so that
\[
B_n = \frac1n \sum_{i=1}^n \{Z_i-\pi_n(X_i)\} h_n(X_i).
\]
Since $\E(B_n)=0$, we have
\begin{align*}
n \E(B_n^2)
&=
\frac1n \sum_{i=1}^n
\E\left[h_n(X_i)^2 \{Z_i-\pi_n(X_i)\}^2\right]  +
\frac{2}{n}\sum_{i<j}
\E\left[h_n(X_i)h_n(X_j)\{Z_i-\pi_n(X_i)\}\{Z_j-\pi_n(X_j)\}\right].
\end{align*}
The diagonal term is again bounded as
\[
\frac1n \sum_{i=1}^n
\E\left[h_n(X_i)^2 \{Z_i-\pi_n(X_i)\}^2\right]
\le
\frac1n \sum_{i=1}^n \E\{h_n(X_i)^2\}
=
\E\{h_n(X_i)^2\},
\]
and the off-diagonal term, by applying Lemma \ref{lem:remainder_var_multi} and $2|ab| \le a^2+b^2$ again, as
\begin{align*}
\frac{2}{n}\sum_{i<j}
\left|\E\left[h_n(X_i)h_n(X_j)\{Z_i-\pi_n(X_i)\}\{Z_j-\pi_n(X_j)\}\right]\right|
\le
C \E\{h_n(X)^2\}.
\end{align*}
Now by positivity,
\[
h_n(X)^2 \le \frac{2}{\gamma^4}\left[\{\var(Y^1\mid X)\}^2+\{\var(Y^0\mid X)\}^2\right].
\]
Moreover $\{\var(Y^z | X)\}^2 \le \E\{(Y^z-\mu_z(X))^4 | X\}$, so the fourth-moment assumption implies $\sup_n \E\{h_n(X)^2\} < \infty$. It follows that $\E(B_n^2) = O(n^{-1})$ and therefore $B_n = o_\bP(1)$ from Chebyshev's inequality. As a result,
\[
\frac{s_n^2}{n} = A_n + B_n \overset{p}{\to} V_1(\theta),
\]
and
\[
\frac{M_n}{s_n^{4}}
=
\frac{M_n/n}{(s_n^2/n)^2} \frac1n \overset{p}{\to} 0,
\]
since $M_n/n = O_\bP(1)$ and $s_n^2/n \overset{p}{\to} V_1(\theta) > 0$. Hence by Lyapunov CLT, we have
\[
T_{1,n} \mid H_n \rightsquigarrow \cN(0,V_1(\theta)).
\]

It remains to show the marginal result
\[
T_{1,n} \rightsquigarrow \cN(0,V_1(\theta)),
\]
and 
\[
    T_{1,n} + T_{2,n} \rightsquigarrow \cN\left(0,V(\theta)\right).
\]

By L\'{e}vy's continuity theorem, it suffices to show pointwise convergence of the corresponding characteristic functions. Define
\[
U_n \equiv \E(e^{itT_{1,n}} \mid H_n) - e^{-t^2V_1(\theta)/2},
\]
and note that
\[
|U_n| \leq \E(|e^{itT_{1,n}}| \mid H_n) + |e^{-t^2V_1(\theta)/2}| \leq 1 + 1 = 2.
\]
Hence, combined with $\E\left(e^{itT_{1,n}} \mid H_n\right) \overset{p}{\to} e^{-t^2V_1(\theta)/2}$, we have $\E|U_n| \to 0$. It follows that
\[
\E(e^{itT_{1,n}})
=
\E\big\{\E(e^{itT_{1,n}} \mid H_n)\big\}
=
e^{-t^2V_1(\theta)/2} + \E(U_n),
\]
and subsequently
\[
\left|\E(e^{itT_{1,n}}) - e^{-t^2V_1(\theta)/2}\right| = |\E(U_n)| \leq \E|U_n| \to 0,
\]
which implies
\[
\E(e^{itT_{1,n}}) \to e^{-t^2V_1(\theta)/2}.
\]

Finally, note that
$\E\big\{e^{it(T_{1,n}+T_{2,n})}\big\} =  e^{t^2V_1(\theta)/2}\E\big(e^{itT_{2,n}}\big)+o(1)$ since
\begin{align*}
    \E\big\{e^{it(T_{1,n} + T_{2,n})}\big\}
    &= \E\big\{e^{itT_{2,n}}\E(e^{itT_{1,n}} \mid H_n)\big\} \\
    &=  \E\big\{e^{itT_{2,n}}(e^{-t^2V_1(\theta)/2} + U_n)\big\} \\
    &= e^{-t^2V_1(\theta)/2}\E\big(e^{itT_{2,n}}\big) + \E\big(e^{itT_{2,n}}U_n\big).
\end{align*}
Moreover, since $|e^{itT_{2,n}}|=1$,
\[
\left|
\E\big\{e^{it(T_{1,n}+T_{2,n})}\big\}
-
e^{-t^2V_1(\theta)/2}\E\big(e^{itT_{2,n}}\big)
\right|
\le
\E|U_n|
\to 0.
\]
Hence, using $\E(e^{itT_{2,n}})\to e^{-t^2V_2/2}$,
\[
    \E(e^{it(T_{1,n} + T_{2,n})}) \to e^{-t^2\{V_1(\theta) + V_2\}/2},
\]
for each $t \in \R$, which concludes the proof.
\end{proof}


\subsection{Proof of Corollary \ref{cor:asymp_normal_psi_Z}}
\begin{proof}
We know from the proof of Theorem \ref{thm:asymp_var_plm} that
\[
\widehat{\psi}_Z-\psi_Z
=
\frac{
\mathbb{P}_n\{f(X,Q)U\}
+
\mathbb{P}_n[f(X,Q)\{m(X)-\widehat m(X)\}]
}{
\mathbb{P}_n[f(X,Q)\{Z-\pi_n(X;\theta,b_n)\}]
}.
\]
By the same logic as in the proof of Theorem \ref{thm:asymp_var_plm}, $\mathbb{P}_n[f(X,Q)\{m(X)-\widehat m(X)\}]$ is \(o_\bP(n^{-1/2})\) and $\mathbb{P}_n[f(X,Q)\{Z-\pi_n(X)\}]$ converges in probability to $\E[f(X_i,Q_i)\{\alpha_{Q_i}(\beta,p)-\pi(X_i)\}]$. Further, since $\E\{f(X_i,Q_i)U_i\}=0$ and $\E(U_i^4)<\infty$, applying Lyapunov CLT gives
\[
\frac{1}{\sqrt{n}}\sum_{i=1}^n f(X_i,Q_i)U_i \rightsquigarrow \cN\!\left(0,\E\{f(X_i,Q_i)^2\sigma(X_i)\}\right).
\]
The result then follows by choosing $f(X_i,Q_i) = r_n(X_i,Q_i)/\widehat\sigma(X_i)$ and Slutsky's theorem.
\end{proof}

\section{Numerical experiments}\label{app:data}

All experiments were conducted on a personal laptop equipped with an Apple M4 chip and 16GB memory. 
All experiments completed in under 15 minutes.

\subsection{Dataset}

The dataset we use in our experiments is from the local government of a large U.S. county. 
The county employs caseworkers who connect high-risk individuals with local organizations that provide housing.
As partner organization units become available, the county refers eligible individuals to the partner.
Many units are reserved for specific populations, such as veterans or women experiencing domestic violence.
In our simulations, we consider individuals only eligible for general-population units.

The county refers individuals to partners using first-in, first-out queues based on a computed risk score $h(X)$, which effectively determines queue assignment. 
The risk score aggregates the output of machine learning algorithms trained to predict several outcomes of interest, such as the risk an individual experiences homelessness in the next six months. 

Individuals in lower queues have effectively no chance at treatment, so we restrict our population of interest to those with a predicted risk score $h(X)$ placing them in the top three queues. 
This ensures all individuals in our sample have some chance at receiving treatment under the county's current allocation process.
After selecting for this population, we have 2,317 individuals in our dataset.

\subsection{Synthetic DGP}

In our dataset, we only observe outcomes under treatment or control, so to calculate the efficiency of our proposed estimators we generate semi-synthetic outcomes using the individual risk score $h(X)$. 
We assume a homogeneous treatment effect $\psi = -0.1$, corresponding to a reduction in risk.
We note $|\psi| \leq \min_{x \in \mathcal X} h(X)$, and so all potential outcome probabilities are well-defined.
Using $\psi$ and $h(X)$, we set the potential outcomes and nuisances as
\begin{align*}
\mu_0(X) &= h(X) \\ 
\mu_1(X) &= h(X) + \psi \\ 
\var(Y \mid Z = 0, X) &= \mu_0(X)(1-\mu_0(X)) \\ 
\var(Y \mid Z = 1, X) &= \mu_1(X)(1-\mu_1(X)).
\end{align*}

\subsection{Sample Efficiency Experiment}\label{app:sample_efficiency}

In this section, we provide implementation details for our two heuristic baseline algorithms and test the robustness of our results to alternative queue configurations.
Similar to the experiment in the main text, we consider both a 25\% and 75\% budget setting, first varying the queue treatment probabilities over equally-proportioned queues, then holding treatment probabilities fixed and varying the queue assignment proportions.

\subsubsection{Heuristic Methods}

We describe our two heuristic methods for constructing queue assignment probabilities given utility scores $u(X)$ and target queue proportions $p \in \Delta_K$.

Our first method, \textit{Switch}, assigns individuals to queues based $u(X)$ and $p$, then assigns each individual a probability of switching to the adjacent higher-priority queue ($p_{\uparrow}$), lower-priority queue ($p_{\downarrow}$), or staying in their assigned queue ($1-p_{\uparrow}-p_{\downarrow}$), returning the final assignment probabilities $\theta(X_i)$.
The switching probabilities $p_{\downarrow}$ and$p_{\uparrow}$ are controlled by a single parameter $\beta \in (0,1)$, scaled by the relative proportions of the adjacent and current queues, so that  larger values of $\beta$ correspond to higher switching probability.

\begin{algorithm}[H]
\caption{Switch}
\begin{algorithmic}[1]
\Require Utilities $u_i$, queue proportions $p \in \Delta_K$, switch probability $\beta \in (0,1)$
\Ensure Assignment probabilities $\theta_{ik}$
\State Compute cumulative proportions: $c_k \gets \sum_{j \le k} p_j$
\State Construct bins: $b_k \gets \mathrm{Quantile}(u; c_k)$ for $k=1,\dots,K-1$
\State Set bin edges: $\tau \gets [-\infty, b_1, \dots, b_{K-1}, \infty]$
\State Assign initial queues: $q_i \gets k \quad \text{such that} \quad \tau_k \leq u_i < \tau_{k+1}$
\For{each individual $i$ with current queue $k = q_i$}
    \State $p_{\uparrow} \gets \beta \cdot \dfrac{p_{k+1}}{p_{k+1} + p_k}$ if $k < K$, else $0$
    \State $p_{\downarrow} \gets \beta \cdot \dfrac{p_{k-1}}{p_{k-1} + p_k}$ if $k > 1$, else $0$
    \If{$p_{\uparrow} + p_{\downarrow} > 1$}
        \State $p_{\uparrow} \gets \dfrac{p_{\uparrow}}{p_{\uparrow} + p_{\downarrow}}, \quad p_{\downarrow} \gets \dfrac{p_{\downarrow}}{p_{\uparrow} + p_{\downarrow}}$
    \EndIf
    \State $\theta_{i,k} \gets 1 - p_{\uparrow} - p_{\downarrow}$
    \State $\theta_{i,k+1} \gets p_{\uparrow}$ \quad (if $k < K$)
    \State $\theta_{i,k-1} \gets p_{\downarrow}$ \quad (if $k > 1$)
\EndFor
\end{algorithmic}
\end{algorithm}

Our second method, \textit{Greedy Softmax}, sequentially constructs assignment probabilities by iterating over queues in priority order. 
At each queue $k$, the algorithm computes a softmax-like score from each individual's utility $u_i$ and residual unallocated probability mass $r_i$, normalizes these scores to meet the target queue assignment proportion $p_k$, and caps individual probabilities via an iterative rescaling procedure that ensures the budget constraint is met.
The remaining mass after processing queues $1, \dots, K-1$ is assigned to the lowest priority queue $K$.
The scale parameter $a$ controls how strongly utilities influence the allocation, with larger values of $a$ concentrating mass in higher-priority queues for individuals with high utility scores.

\begin{algorithm}[H]
\caption{Greedy Softmax}
\begin{algorithmic}[1]
\Require Utilities $u_i$, target proportions $p \in \Delta_K$, scale $a$, maximum value $m$
\Ensure Assignment probabilities $\theta_{ik}$
\State Initialize residual weights: $r_i \gets 1$
\State Initialize $\theta_{ik} \gets 0$
\For{$k = 1, 2, \dots, K - 1$}
    \State Compute scores: $s_i \gets \exp(a \cdot u_i \cdot r_i) - 1$
    \State Normalize: $s_i \gets s_i \cdot \frac{p_k n}{\sum_i s_i}$
    \While{$\max_i s_i > m$ and $|\mathcal{C}| \cdot m < p_k n$}
        \State $i^* \gets \arg\max_i s_i$
        \State Add $i^*$ to capped set $\mathcal{C}$, set $s_{i^*} \gets 0$
        \State Rescale: $s_i \gets s_i \cdot \frac{p_k n - |\mathcal{C}| \cdot m}{\sum_i s_i}$ for $i \notin \mathcal{C}$
    \EndWhile
    \State Set $s_i \gets m$ for $i \in \mathcal{C}$
    \State Set $\theta_{ik} \gets s_i$
    \State Update residuals: $r_i \gets r_i - \theta_{ik}$
\EndFor
\State Assign remaining mass: $\theta_{iK} \gets r_i$
\end{algorithmic}
\end{algorithm}

\subsubsection{Varying Budget}

We study queue configurations analogous to those in the main text under 25\% and 75\% budget scenarios. 
Figures~\ref{app:fig:25budget} and \ref{app:fig:75budget} reveal qualitatively similar patterns, though the endogenous arrival estimator is less efficient due to lower variation in $\alpha_Q$ We examine this more directly in the next section, where we explicitly vary $\alpha_Q$ and study its effect on efficiency.

\begin{figure}[H]
    \centering
    \includegraphics[width=\textwidth]{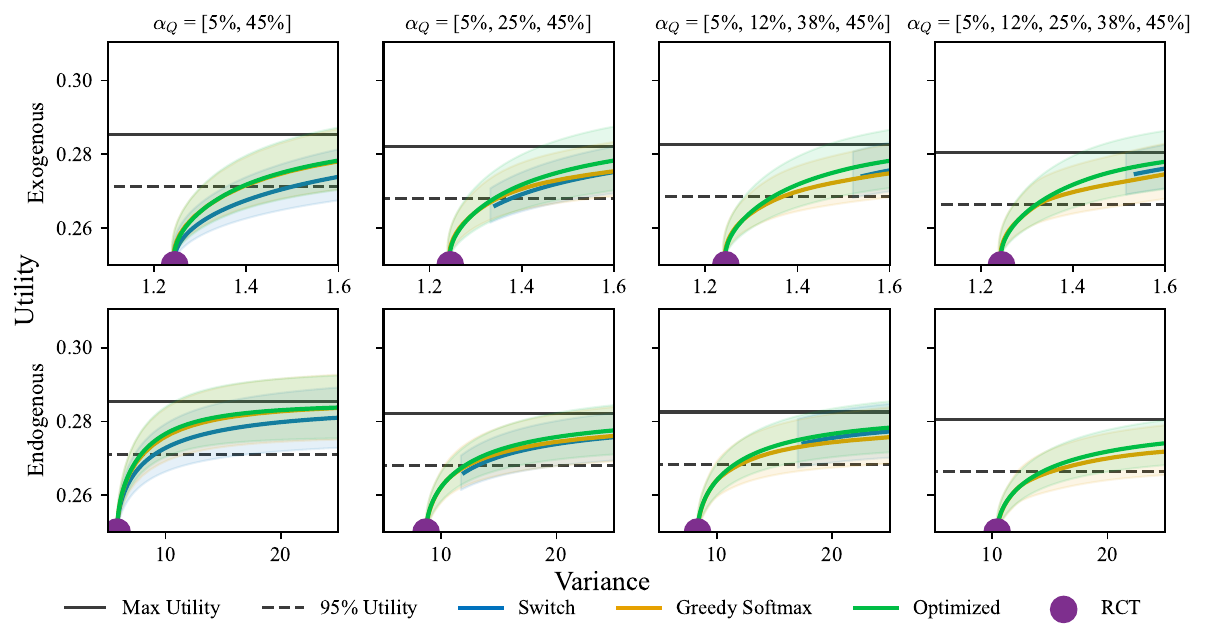}
    \caption{Pareto frontier of variance and utility under our optimized designs (green) versus two heuristic designs (orange, blue) and an RCT baseline (purple) under exogenous arrivals (top) and endogenous arrivals (bottom) across different queue configurations with a 25\% treatment budget. Queues are equally sized but vary in number, with treatment probabilities shown at the top of each column.}
    \label{app:fig:25budget}
\end{figure}

\begin{figure}[H]
    \centering
    \includegraphics[width=\textwidth]{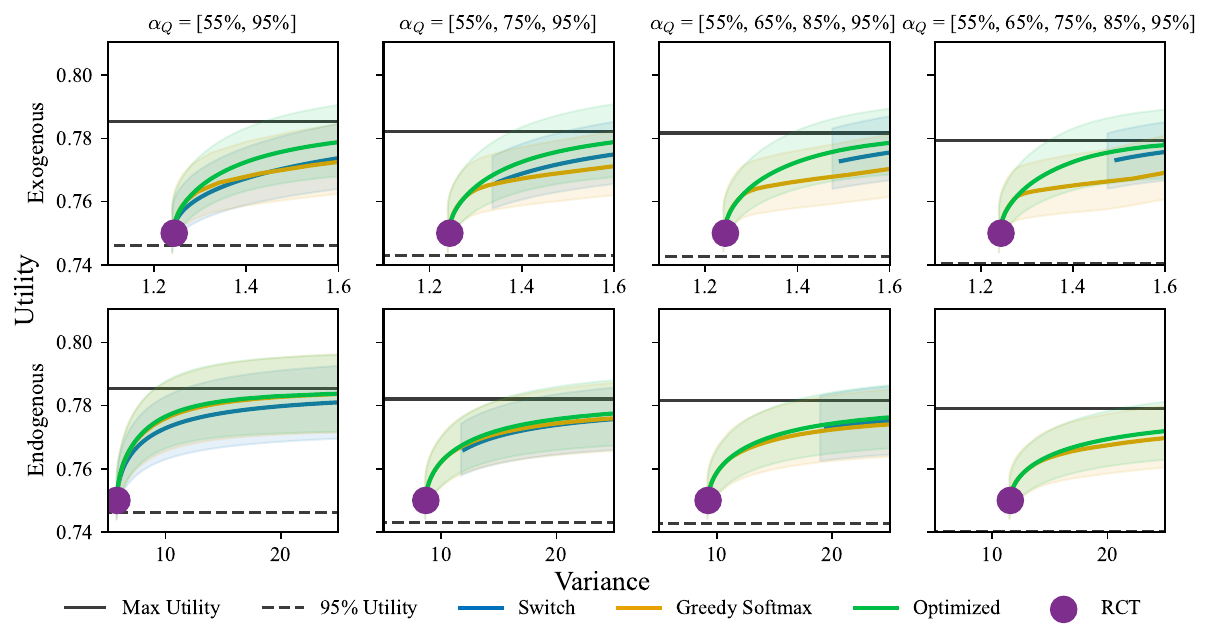}
    \caption{Pareto frontier of variance and utility under our optimized designs (green) versus two heuristic designs (orange, blue) and an RCT baseline (purple) under exogenous arrivals (top) and endogenous arrivals (bottom) across different queue configurations with a 50\% treatment budget. Queues are equally sized but vary in number, with treatment probabilities shown at the top of each column. All designs, including the RCT, achieve more than 95\% utility.}
    \label{app:fig:75budget}
\end{figure}

\subsubsection{Variance of $\alpha_Q$ on Efficiency}

Figure~\ref{app:fig:increasingvar} shows that the between-queue variance of treatment probabilities greatly influences the efficiency of the endogenous estimator. As the variance of $\alpha_Q$ increases, the endogenous estimator becomes more efficient, and at extreme treatment probabilities, 1\% in the lowest-priority queue and 99\% in the highest-priority queue, it achieves efficiency comparable to the exogenous estimator.

\begin{figure}[H]
    \centering
    \includegraphics[width=\textwidth]{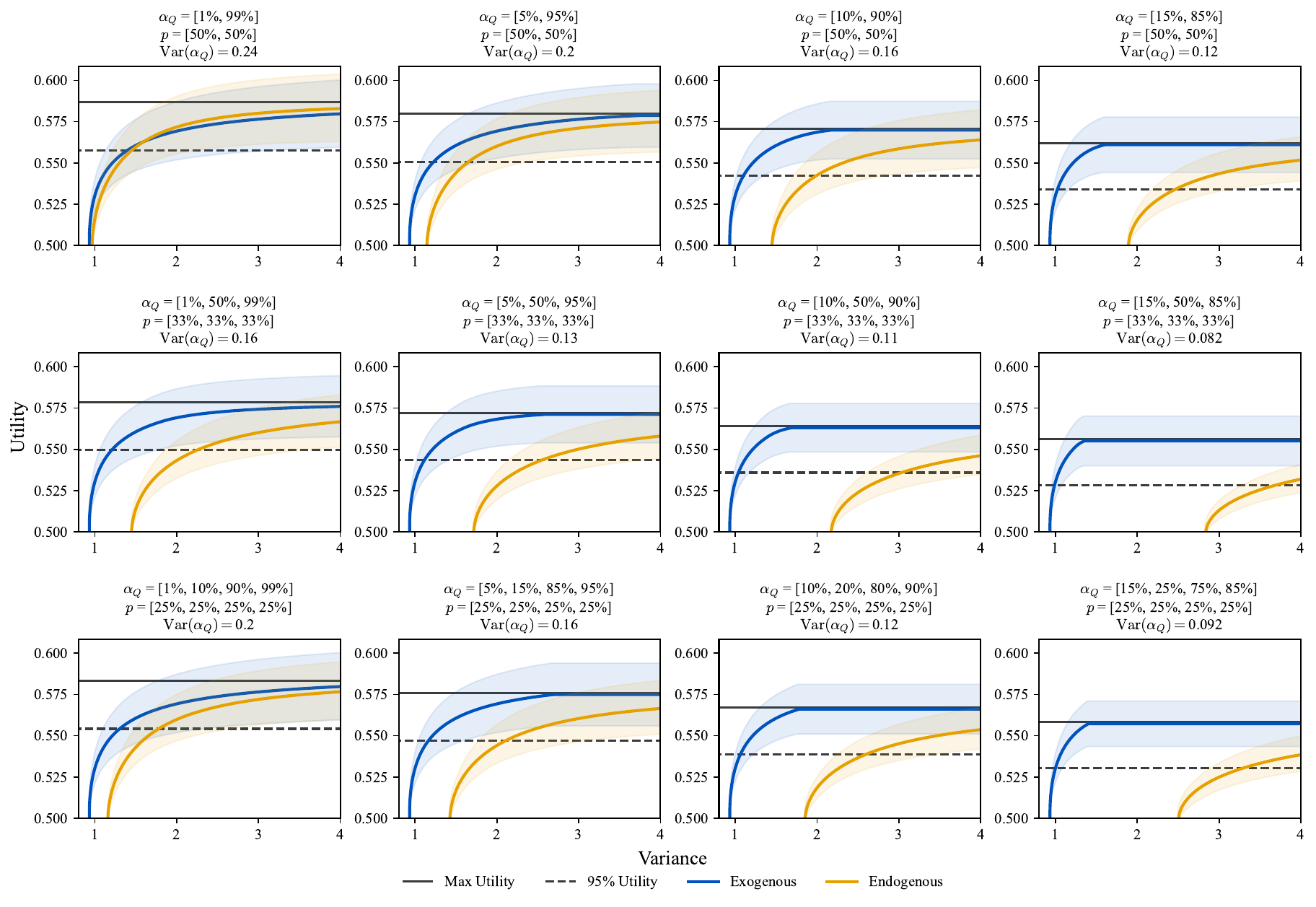}
    \caption{Pareto frontier of variance and utility under the endogenous design (orange) and exogenous design (blue) across different queue configurations with a 50\% treatment budget. Treatment probabilities and queue assignment proportions are shown at the top of each column.}
    \label{app:fig:increasingvar}
\end{figure}

\subsubsection{Varying Queue Assignment Proportions}

We vary the queue proportions while keeping the budget fixed, allocating more individuals to lower-priority queues. As with the varying budget setting, Figure~\ref{app:fig:queue_assignment_props} revals that the optimized designs remain competitive across all queue configurations, while neither heuristic dominates the other.

\begin{figure}[H]
    \centering
    \includegraphics[width=\textwidth]{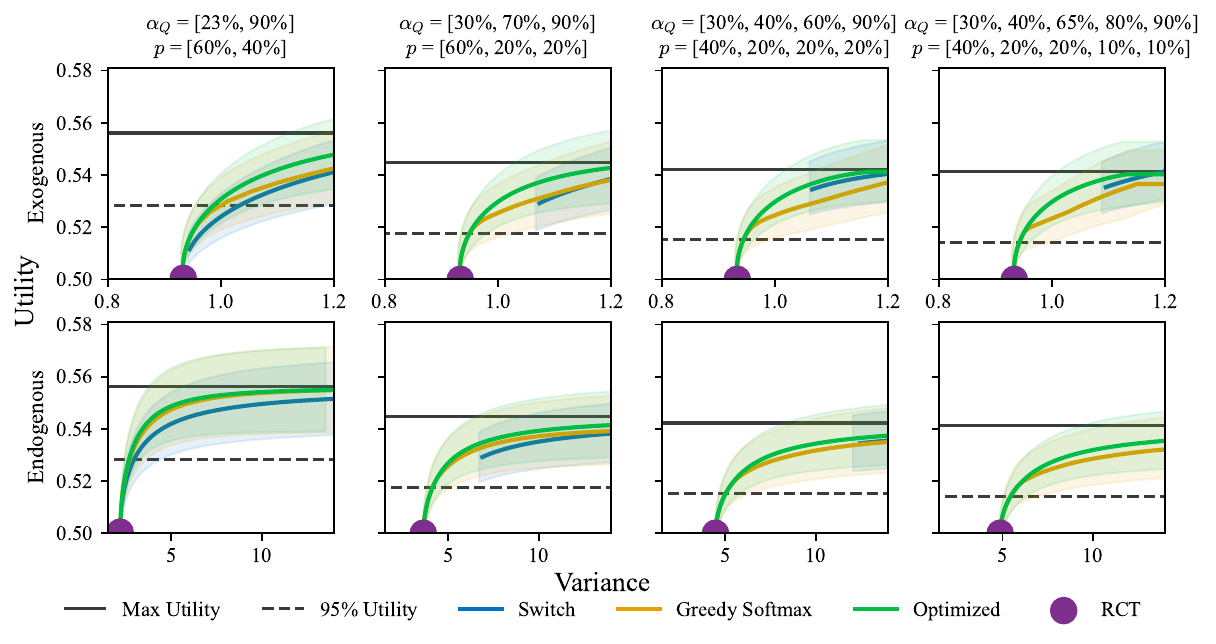}
    \caption{Pareto frontier of variance and utility under our optimized designs (green) versus two heuristic designs (orange, blue) and an RCT baseline (purple) under exogenous arrivals (top) and endogenous arrivals (bottom) across different queue configurations with a 50\% treatment budget. Treatment probabilities and queue assignment proportions are shown at the top of each column.}
    \label{app:fig:queue_assignment_props}
\end{figure}

\subsection{Bias Experiment}\label{app:bias}

Potential are generated according to the partially linear model
\begin{align*}
    Y^1 &= \psi + h(X) + U, \\
    Y^0 &= h(X) + U,
\end{align*}
where $\psi = -0.1$ is the homogeneous treatment effect and 
\begin{align*}
    U \sim \text{Uniform}(-0.2h(X), \, 0.2 h(X)),
\end{align*}
The arrival $A$ is set as the indices of $U$ sorted from highest to lowest.

For both the endogenous and exogenous settings, we use $A$, $Y^1$, $Y^0$, and the corresponding $\theta(X)$, to simulate observed treatment assignments $Z$ and outcomes $Y$. 
To study bias, we also calculate the misspecified nuisance functions $E[Y^1 \mid X]$ and $E[Y^0 \mid X]$ that omit $U$, introducing misspecification into the exogenous estimator $\psi_{\text{DR}}$, which relies on these models.
The realized outcomes, treatments, and known nuisance functions are then plugged into their corresponding estimators $\psi_{DR}$ in the exogenous setting and $\psi_Z$ in the endogenous setting. 
We repeat this procedure 10,000 times and report the average bias $\E[\widehat \psi_{\cdot}] - \psi$ for both estimators.


\newpage
\section*{NeurIPS Paper Checklist}

The checklist is designed to encourage best practices for responsible machine learning research, addressing issues of reproducibility, transparency, research ethics, and societal impact. Do not remove the checklist: {\bf The papers not including the checklist will be desk rejected.} The checklist should follow the references and follow the (optional) supplemental material.  The checklist does NOT count towards the page
limit. 

Please read the checklist guidelines carefully for information on how to answer these questions. For each question in the checklist:
\begin{itemize}
    \item You should answer \answerYes{}, \answerNo{}, or \answerNA{}.
    \item \answerNA{} means either that the question is Not Applicable for that particular paper or the relevant information is Not Available.
    \item Please provide a short (1--2 sentence) justification right after your answer (even for \answerNA). 
\end{itemize}

{\bf The checklist answers are an integral part of your paper submission.} They are visible to the reviewers, area chairs, senior area chairs, and ethics reviewers. You will also be asked to include it (after eventual revisions) with the final version of your paper, and its final version will be published with the paper.

The reviewers of your paper will be asked to use the checklist as one of the factors in their evaluation. While \answerYes{} is generally preferable to \answerNo{}, it is perfectly acceptable to answer \answerNo{} provided a proper justification is given (e.g., error bars are not reported because it would be too computationally expensive'' or ``we were unable to find the license for the dataset we used''). In general, answering \answerNo{} or \answerNA{} is not grounds for rejection. While the questions are phrased in a binary way, we acknowledge that the true answer is often more nuanced, so please just use your best judgment and write a justification to elaborate. All supporting evidence can appear either in the main paper or the supplemental material, provided in appendix. If you answer \answerYes{} to a question, in the justification please point to the section(s) where related material for the question can be found.

IMPORTANT, please:
\begin{itemize}
    \item {\bf Delete this instruction block, but keep the section heading ``NeurIPS Paper Checklist"},
    \item  {\bf Keep the checklist subsection headings, questions/answers and guidelines below.}
    \item {\bf Do not modify the questions and only use the provided macros for your answers}.
\end{itemize}


\begin{enumerate}

\item {\bf Claims}
    \item[] Question: Do the main claims made in the abstract and introduction accurately reflect the paper's contributions and scope?
    \item[] Answer: \answerYes{}
    \item[] Justification: All claims in the abstract and introduction are supported by main text.
    \item[] Guidelines:
    \begin{itemize}
        \item The answer \answerNA{} means that the abstract and introduction do not include the claims made in the paper.
        \item The abstract and/or introduction should clearly state the claims made, including the contributions made in the paper and important assumptions and limitations. A \answerNo{} or \answerNA{} answer to this question will not be perceived well by the reviewers. 
        \item The claims made should match theoretical and experimental results, and reflect how much the results can be expected to generalize to other settings. 
        \item It is fine to include aspirational goals as motivation as long as it is clear that these goals are not attained by the paper. 
    \end{itemize}

\item {\bf Limitations}
    \item[] Question: Does the paper discuss the limitations of the work performed by the authors?
    \item[] Answer: \answerYes{} 
    \item[] Justification: We discuss limitations and potential extensions in Section \ref{sec:conclusion}.
    \item[] Guidelines:
    \begin{itemize}
        \item The answer \answerNA{} means that the paper has no limitation while the answer \answerNo{} means that the paper has limitations, but those are not discussed in the paper. 
        \item The authors are encouraged to create a separate ``Limitations'' section in their paper.
        \item The paper should point out any strong assumptions and how robust the results are to violations of these assumptions (e.g., independence assumptions, noiseless settings, model well-specification, asymptotic approximations only holding locally). The authors should reflect on how these assumptions might be violated in practice and what the implications would be.
        \item The authors should reflect on the scope of the claims made, e.g., if the approach was only tested on a few datasets or with a few runs. In general, empirical results often depend on implicit assumptions, which should be articulated.
        \item The authors should reflect on the factors that influence the performance of the approach. For example, a facial recognition algorithm may perform poorly when image resolution is low or images are taken in low lighting. Or a speech-to-text system might not be used reliably to provide closed captions for online lectures because it fails to handle technical jargon.
        \item The authors should discuss the computational efficiency of the proposed algorithms and how they scale with dataset size.
        \item If applicable, the authors should discuss possible limitations of their approach to address problems of privacy and fairness.
        \item While the authors might fear that complete honesty about limitations might be used by reviewers as grounds for rejection, a worse outcome might be that reviewers discover limitations that aren't acknowledged in the paper. The authors should use their best judgment and recognize that individual actions in favor of transparency play an important role in developing norms that preserve the integrity of the community. Reviewers will be specifically instructed to not penalize honesty concerning limitations.
    \end{itemize}

\item {\bf Theory assumptions and proofs}
    \item[] Question: For each theoretical result, does the paper provide the full set of assumptions and a complete (and correct) proof?
    \item[] Answer: \answerYes{} 
    \item[] Justification: We provide full set of assumptions in Sections \ref{sec:identification} and \ref{sec:design} and the complete proofs in Appendix \ref{sec:proofs}.
    \item[] Guidelines:
    \begin{itemize}
        \item The answer \answerNA{} means that the paper does not include theoretical results. 
        \item All the theorems, formulas, and proofs in the paper should be numbered and cross-referenced.
        \item All assumptions should be clearly stated or referenced in the statement of any theorems.
        \item The proofs can either appear in the main paper or the supplemental material, but if they appear in the supplemental material, the authors are encouraged to provide a short proof sketch to provide intuition. 
        \item Inversely, any informal proof provided in the core of the paper should be complemented by formal proofs provided in appendix or supplemental material.
        \item Theorems and Lemmas that the proof relies upon should be properly referenced. 
    \end{itemize}

    \item {\bf Experimental result reproducibility}
    \item[] Question: Does the paper fully disclose all the information needed to reproduce the main experimental results of the paper to the extent that it affects the main claims and/or conclusions of the paper (regardless of whether the code and data are provided or not)?
    \item[] Answer: \answerYes{} 
    \item[] Justification: We are not able to provide the original data due to its sensitive nature, but we provide detailed information on all experiments in the Appendix~\ref{app:sample_efficiency}and ~\ref{app:bias}. We also provide a demo notebook with a fully synthetic dataset that demonstrates our approach.
    \item[] Guidelines:
    \begin{itemize}
        \item The answer \answerNA{} means that the paper does not include experiments.
        \item If the paper includes experiments, a \answerNo{} answer to this question will not be perceived well by the reviewers: Making the paper reproducible is important, regardless of whether the code and data are provided or not.
        \item If the contribution is a dataset and\slash or model, the authors should describe the steps taken to make their results reproducible or verifiable. 
        \item Depending on the contribution, reproducibility can be accomplished in various ways. For example, if the contribution is a novel architecture, describing the architecture fully might suffice, or if the contribution is a specific model and empirical evaluation, it may be necessary to either make it possible for others to replicate the model with the same dataset, or provide access to the model. In general. releasing code and data is often one good way to accomplish this, but reproducibility can also be provided via detailed instructions for how to replicate the results, access to a hosted model (e.g., in the case of a large language model), releasing of a model checkpoint, or other means that are appropriate to the research performed.
        \item While NeurIPS does not require releasing code, the conference does require all submissions to provide some reasonable avenue for reproducibility, which may depend on the nature of the contribution. For example
        \begin{enumerate}
            \item If the contribution is primarily a new algorithm, the paper should make it clear how to reproduce that algorithm.
            \item If the contribution is primarily a new model architecture, the paper should describe the architecture clearly and fully.
            \item If the contribution is a new model (e.g., a large language model), then there should either be a way to access this model for reproducing the results or a way to reproduce the model (e.g., with an open-source dataset or instructions for how to construct the dataset).
            \item We recognize that reproducibility may be tricky in some cases, in which case authors are welcome to describe the particular way they provide for reproducibility. In the case of closed-source models, it may be that access to the model is limited in some way (e.g., to registered users), but it should be possible for other researchers to have some path to reproducing or verifying the results.
        \end{enumerate}
    \end{itemize}

\item {\bf Open access to data and code}
    \item[] Question: Does the paper provide open access to the data and code, with sufficient instructions to faithfully reproduce the main experimental results, as described in supplemental material?
    \item[] Answer: \answerNo{} 
    \item[] Justification: The dataset contains sensitive information about individuals and so we are unable to release it. We do provide code and a demo notebook with a fully synthetic dataset that demonstrates our approach.
    \item[] Guidelines:
    \begin{itemize}
        \item The answer \answerNA{} means that paper does not include experiments requiring code.
        \item Please see the NeurIPS code and data submission guidelines (\url{https://neurips.cc/public/guides/CodeSubmissionPolicy}) for more details.
        \item While we encourage the release of code and data, we understand that this might not be possible, so \answerNo{} is an acceptable answer. Papers cannot be rejected simply for not including code, unless this is central to the contribution (e.g., for a new open-source benchmark).
        \item The instructions should contain the exact command and environment needed to run to reproduce the results. See the NeurIPS code and data submission guidelines (\url{https://neurips.cc/public/guides/CodeSubmissionPolicy}) for more details.
        \item The authors should provide instructions on data access and preparation, including how to access the raw data, preprocessed data, intermediate data, and generated data, etc.
        \item The authors should provide scripts to reproduce all experimental results for the new proposed method and baselines. If only a subset of experiments are reproducible, they should state which ones are omitted from the script and why.
        \item At submission time, to preserve anonymity, the authors should release anonymized versions (if applicable).
        \item Providing as much information as possible in supplemental material (appended to the paper) is recommended, but including URLs to data and code is permitted.
    \end{itemize}

\item {\bf Experimental setting/details}
    \item[] Question: Does the paper specify all the training and test details (e.g., data splits, hyperparameters, how they were chosen, type of optimizer) necessary to understand the results?
    \item[] Answer: \answerYes{} 
    \item[] Justification: We provide detailed descriptions on all experiments in the Appendix~\ref{app:sample_efficiency} and \ref{app:bias}.
    \item[] Guidelines:
    \begin{itemize}
        \item The answer \answerNA{} means that the paper does not include experiments.
        \item The experimental setting should be presented in the core of the paper to a level of detail that is necessary to appreciate the results and make sense of them.
        \item The full details can be provided either with the code, in appendix, or as supplemental material.
    \end{itemize}

\item {\bf Experiment statistical significance}
    \item[] Question: Does the paper report error bars suitably and correctly defined or other appropriate information about the statistical significance of the experiments?
    \item[] Answer: \answerYes{} 
    \item[] Justification: As explained in the main text, we use the multiplier bootstrap for uncertainty quantification in Section \ref{sec:experiments} and in the Appendix~\ref{app:sample_efficiency}.
    \item[] Guidelines:
    \begin{itemize}
        \item The answer \answerNA{} means that the paper does not include experiments.
        \item The authors should answer \answerYes{} if the results are accompanied by error bars, confidence intervals, or statistical significance tests, at least for the experiments that support the main claims of the paper.
        \item The factors of variability that the error bars are capturing should be clearly stated (for example, train/test split, initialization, random drawing of some parameter, or overall run with given experimental conditions).
        \item The method for calculating the error bars should be explained (closed form formula, call to a library function, bootstrap, etc.)
        \item The assumptions made should be given (e.g., Normally distributed errors).
        \item It should be clear whether the error bar is the standard deviation or the standard error of the mean.
        \item It is OK to report 1-sigma error bars, but one should state it. The authors should preferably report a 2-sigma error bar than state that they have a 96\% CI, if the hypothesis of Normality of errors is not verified.
        \item For asymmetric distributions, the authors should be careful not to show in tables or figures symmetric error bars that would yield results that are out of range (e.g., negative error rates).
        \item If error bars are reported in tables or plots, the authors should explain in the text how they were calculated and reference the corresponding figures or tables in the text.
    \end{itemize}

\item {\bf Experiments compute resources}
    \item[] Question: For each experiment, does the paper provide sufficient information on the computer resources (type of compute workers, memory, time of execution) needed to reproduce the experiments?
    \item[] Answer: \answerYes{}{} 
    \item[] Justification: We state that all experiments were run on a personal laptop with a M4 chip and 16gb of memory in Appendix~\ref{app:data}.
    \item[] Guidelines:
    \begin{itemize}
        \item The answer \answerNA{} means that the paper does not include experiments.
        \item The paper should indicate the type of compute workers CPU or GPU, internal cluster, or cloud provider, including relevant memory and storage.
        \item The paper should provide the amount of compute required for each of the individual experimental runs as well as estimate the total compute. 
        \item The paper should disclose whether the full research project required more compute than the experiments reported in the paper (e.g., preliminary or failed experiments that didn't make it into the paper). 
    \end{itemize}
    
\item {\bf Code of ethics}
    \item[] Question: Does the research conducted in the paper conform, in every respect, with the NeurIPS Code of Ethics \url{https://neurips.cc/public/EthicsGuidelines}?
    \item[] Answer: \answerYes{} 
    \item[] Justification: We have reviewed the Code of Ethics and confirm that our papers complies with it.
    \item[] Guidelines:
    \begin{itemize}
        \item The answer \answerNA{} means that the authors have not reviewed the NeurIPS Code of Ethics.
        \item If the authors answer \answerNo, they should explain the special circumstances that require a deviation from the Code of Ethics.
        \item The authors should make sure to preserve anonymity (e.g., if there is a special consideration due to laws or regulations in their jurisdiction).
    \end{itemize}

\item {\bf Broader impacts}
    \item[] Question: Does the paper discuss both potential positive societal impacts and negative societal impacts of the work performed?
    \item[] Answer: \answerYes{} 
    \item[] Justification: We discuss potential societal impacts in Section \ref{sec:conclusion}.
    \item[] Guidelines:
    \begin{itemize}
        \item The answer \answerNA{} means that there is no societal impact of the work performed.
        \item If the authors answer \answerNA{} or \answerNo, they should explain why their work has no societal impact or why the paper does not address societal impact.
        \item Examples of negative societal impacts include potential malicious or unintended uses (e.g., disinformation, generating fake profiles, surveillance), fairness considerations (e.g., deployment of technologies that could make decisions that unfairly impact specific groups), privacy considerations, and security considerations.
        \item The conference expects that many papers will be foundational research and not tied to particular applications, let alone deployments. However, if there is a direct path to any negative applications, the authors should point it out. For example, it is legitimate to point out that an improvement in the quality of generative models could be used to generate Deepfakes for disinformation. On the other hand, it is not needed to point out that a generic algorithm for optimizing neural networks could enable people to train models that generate Deepfakes faster.
        \item The authors should consider possible harms that could arise when the technology is being used as intended and functioning correctly, harms that could arise when the technology is being used as intended but gives incorrect results, and harms following from (intentional or unintentional) misuse of the technology.
        \item If there are negative societal impacts, the authors could also discuss possible mitigation strategies (e.g., gated release of models, providing defenses in addition to attacks, mechanisms for monitoring misuse, mechanisms to monitor how a system learns from feedback over time, improving the efficiency and accessibility of ML).
    \end{itemize}
    
\item {\bf Safeguards}
    \item[] Question: Does the paper describe safeguards that have been put in place for responsible release of data or models that have a high risk for misuse (e.g., pre-trained language models, image generators, or scraped datasets)?
    \item[] Answer: \answerNA{} 
    \item[] Justification: Our work does not contain any data or models that have a high risk for misuse. 
    \item[] Guidelines:
    \begin{itemize}
        \item The answer \answerNA{} means that the paper poses no such risks.
        \item Released models that have a high risk for misuse or dual-use should be released with necessary safeguards to allow for controlled use of the model, for example by requiring that users adhere to usage guidelines or restrictions to access the model or implementing safety filters. 
        \item Datasets that have been scraped from the Internet could pose safety risks. The authors should describe how they avoided releasing unsafe images.
        \item We recognize that providing effective safeguards is challenging, and many papers do not require this, but we encourage authors to take this into account and make a best faith effort.
    \end{itemize}

\item {\bf Licenses for existing assets}
    \item[] Question: Are the creators or original owners of assets (e.g., code, data, models), used in the paper, properly credited and are the license and terms of use explicitly mentioned and properly respected?
    \item[] Answer: \answerYes{} 
    \item[] Justification: We use a proprietary dataset from the county that will be properly attributed, but is not right now to preserve anonymity.
    \item[] Guidelines:
    \begin{itemize}
        \item The answer \answerNA{} means that the paper does not use existing assets.
        \item The authors should cite the original paper that produced the code package or dataset.
        \item The authors should state which version of the asset is used and, if possible, include a URL.
        \item The name of the license (e.g., CC-BY 4.0) should be included for each asset.
        \item For scraped data from a particular source (e.g., website), the copyright and terms of service of that source should be provided.
        \item If assets are released, the license, copyright information, and terms of use in the package should be provided. For popular datasets, \url{paperswithcode.com/datasets} has curated licenses for some datasets. Their licensing guide can help determine the license of a dataset.
        \item For existing datasets that are re-packaged, both the original license and the license of the derived asset (if it has changed) should be provided.
        \item If this information is not available online, the authors are encouraged to reach out to the asset's creators.
    \end{itemize}

\item {\bf New assets}
    \item[] Question: Are new assets introduced in the paper well documented and is the documentation provided alongside the assets?
    \item[] Answer: \answerYes{} 
    \item[] Justification: We provide a starter notebook and code to run our approach in the supplementary materials.
    \item[] Guidelines:
    \begin{itemize}
        \item The answer \answerNA{} means that the paper does not release new assets.
        \item Researchers should communicate the details of the dataset\slash code\slash model as part of their submissions via structured templates. This includes details about training, license, limitations, etc. 
        \item The paper should discuss whether and how consent was obtained from people whose asset is used.
        \item At submission time, remember to anonymize your assets (if applicable). You can either create an anonymized URL or include an anonymized zip file.
    \end{itemize}

\item {\bf Crowdsourcing and research with human subjects}
    \item[] Question: For crowdsourcing experiments and research with human subjects, does the paper include the full text of instructions given to participants and screenshots, if applicable, as well as details about compensation (if any)? 
    \item[] Answer: \answerNA{} 
    \item[] Justification: Our work does not regard any crowdsourcing experiments and research with human subjects.
    \item[] Guidelines:
    \begin{itemize}
        \item The answer \answerNA{} means that the paper does not involve crowdsourcing nor research with human subjects.
        \item Including this information in the supplemental material is fine, but if the main contribution of the paper involves human subjects, then as much detail as possible should be included in the main paper. 
        \item According to the NeurIPS Code of Ethics, workers involved in data collection, curation, or other labor should be paid at least the minimum wage in the country of the data collector. 
    \end{itemize}

\item {\bf Institutional review board (IRB) approvals or equivalent for research with human subjects}
    \item[] Question: Does the paper describe potential risks incurred by study participants, whether such risks were disclosed to the subjects, and whether Institutional Review Board (IRB) approvals (or an equivalent approval/review based on the requirements of your country or institution) were obtained?
    \item[] Answer: \answerNA{} 
    \item[] Justification: Our work does not regard any crowdsourcing experiments and research with human subjects.
    \item[] Guidelines:
    \begin{itemize}
        \item The answer \answerNA{} means that the paper does not involve crowdsourcing nor research with human subjects.
        \item Depending on the country in which research is conducted, IRB approval (or equivalent) may be required for any human subjects research. If you obtained IRB approval, you should clearly state this in the paper. 
        \item We recognize that the procedures for this may vary significantly between institutions and locations, and we expect authors to adhere to the NeurIPS Code of Ethics and the guidelines for their institution. 
        \item For initial submissions, do not include any information that would break anonymity (if applicable), such as the institution conducting the review.
    \end{itemize}

\item {\bf Declaration of LLM usage}
    \item[] Question: Does the paper describe the usage of LLMs if it is an important, original, or non-standard component of the core methods in this research? Note that if the LLM is used only for writing, editing, or formatting purposes and does \emph{not} impact the core methodology, scientific rigor, or originality of the research, declaration is not required.
    \item[] Answer: \answerNA{} 
    \item[] Justification: The core method development in our work does not involve LLMs as any important, original, or non-standard components.
    \item[] Guidelines:
    \begin{itemize}
        \item The answer \answerNA{} means that the core method development in this research does not involve LLMs as any important, original, or non-standard components.
        \item Please refer to our LLM policy in the NeurIPS handbook for what should or should not be described.
    \end{itemize}

\end{enumerate}

\end{document}